\documentclass{article}

\usepackage[preprint]{neurips_2026}

\usepackage{graphicx}
\usepackage{subfig}
\usepackage{float}
\usepackage{placeins}
\usepackage[linesnumbered,ruled]{algorithm2e}
\usepackage{bm}
\usepackage{amsmath}
\usepackage{amssymb}
\usepackage{mathtools}
\usepackage{amsthm}
\usepackage{ifthen}
\usepackage{pifont} 
\usepackage{tabularx}
\usepackage{threeparttable}
\usepackage{array}

\newcommand{\compilehidecomments}{false} 
\ifthenelse{ \equal{\compilehidecomments}{true} }{%
	\newcommand{\JZ}[1]{}
    \newcommand{\QR}[1]{}
}{
\newcommand{\JZ}[1]{{\color{orange} [\text{JZ:} #1]}}
\newcommand{\QR}[1]{{\color{violet} [\text{QR:} #1]}}
}

\usepackage[utf8]{inputenc} 
\usepackage[T1]{fontenc}    
\usepackage[
    colorlinks=true,
    linkcolor=blue,
    citecolor=blue,
    urlcolor=blue,
    filecolor=blue,
    bookmarks=true,
    bookmarksopen=true,
    bookmarksnumbered=true,
    pdfstartview=FitH,
    breaklinks=true
]{hyperref}                 
\usepackage{url}            
\usepackage{booktabs}       
\usepackage{amsfonts}       
\usepackage{nicefrac}       
\usepackage[expansion=false]{microtype}
\usepackage{xcolor}         
\usepackage{tikz}
\usepackage{pgfplots}
\usepgfplotslibrary{fillbetween}
\pgfplotsset{compat=1.18}
\usepackage{xpatch}

\makeatletter
\xapptocmd{\NAT@bibsetnum}{\setlength{\leftmargin}{0pt}\setlength{\itemindent}{\labelwidth}\addtolength{\itemindent}{\labelsep}}{}{}
\makeatother

\usepackage[capitalize,noabbrev]{cleveref}
\theoremstyle{plain}
\newtheorem{theorem}{Theorem}[section]

\newtheorem{lemma}[theorem]{Lemma}
\newtheorem{corollary}[theorem]{Corollary}

\crefname{proposition}{proposition}{propositions}
\Crefname{proposition}{Proposition}{Propositions}
\crefname{lemma}{lemma}{lemmas}
\Crefname{lemma}{Lemma}{Lemmas}
\crefname{corollary}{corollary}{corollaries}
\Crefname{corollary}{Corollary}{Corollaries}

\theoremstyle{definition}
\newtheorem{definition}[theorem]{Definition}
\newtheorem{assumption}[theorem]{Assumption}

\crefname{definition}{definition}{definitions}
\Crefname{definition}{Definition}{Definitions}
\crefname{assumption}{assumption}{assumptions}
\Crefname{assumption}{Assumption}{Assumptions}

\theoremstyle{remark}
\newtheorem{remark}[theorem]{Remark}

\crefname{remark}{remark}{remarks}
\Crefname{remark}{Remark}{Remarks}

\newcommand{\ru}{r_{\max}}
\newcommand{\rl}{r_{\min}}
\newcommand{\on}{\mathrm{on}}
\newcommand{\off}{\mathrm{off}}

\title{Beyond Reward Suppression:\\ Near-Optimal Offline Attacks on Warm-Start Bandits with Bounded Rewards}

\author{%
{\bfseries
Qirun Zeng$^{1}$ \quad
Manhin Poon$^{1}$ \quad
Xiangxiang Dai$^{2}$ \quad
Qixin Zhang$^{3}$ \quad
Jinhang Zuo$^{1}$} \\[0.7em]
{\small
$^{1}$City University of Hong Kong, Hong Kong SAR, China} \\[-0.05em]
{\small
$^{2}$The Chinese University of Hong Kong, Hong Kong SAR, China} \\[-0.05em]
{\small
$^{3}$Nanyang Technological University, Singapore} \\[0.4em]
{\small
\texttt{\{qirun.zeng,manhpoon4-c\}@my.cityu.edu.hk}} \\[-0.05em]
{\small
\texttt{xxdai23@cse.cuhk.edu.hk}
\quad
\texttt{qixin.zhang@ntu.edu.sg}} \\[-0.05em]
{\small
\texttt{jinhang.zuo@cityu.edu.hk}}
}

\begin{document}

\maketitle

\begin{abstract}
Adversarial attacks on bandits aim to mislead a learner toward a target arm while keeping the attack cost small. Existing attacks typically achieve this by suppressing non-target arms. In practice, however, manipulation such as fake reviews often directly promotes the target item. We study this gap through bounded offline attacks on warm-start bandits, where an attacker can inject only valid action–reward pairs into the warm-start history before deployment. We show that target promotion is not merely a heuristic: when the target arm lies near the lower reward boundary, any order-optimal-cost attack against UCB that makes it selected in nearly all online rounds must allocate a nonvanishing fraction of its cost to the target arm.
We then design an attack that achieves the optimal sublinear cost and characterize its allocation between target promotion and non-target suppression. We further extend the attack to Thompson Sampling, $\epsilon$-greedy, and a broader class of bandit algorithms. Experiments on real-world and synthetic data validate the effectiveness of our attacks.
\end{abstract}

\section{Introduction}
Multi-armed bandit (MAB) algorithms provide a principled framework for sequential decision-making under uncertainty and have been widely adopted in applications such as recommender systems and online advertising. In many practical applications, however, bandit algorithms are not deployed from scratch. Instead, they are initialized with historical interaction data, giving rise to the warm-start bandit setting~\citep{shivaswamy2012multi,li2021unifying,banerjee2022artificial,oetomo2023cutting}. 
While such historical data can improve early decisions and reduce exploration costs, they also create a natural attack surface: fake feedback injected before deployment can bias the learner's initial estimates, and this bias may persist even when all subsequent feedback is genuine.

Adversarial attacks on bandits have been studied extensively, but most existing work~\citep{jun2018adversarial,liu2019data,zuo2024near} focuses on manipulating rewards during online interaction, typically allowing arbitrary modifications to observed rewards and measuring attack cost by the total magnitude of these modifications. \citet{zeng2025practical} takes a step further by injecting fake action--reward records to model fake users. Building on these works, we consider a more natural setting in which the learner is initialized with offline logs before online deployment, and the attacker can inject fake action--reward records only during this offline stage. We further measure attack cost by the number of injected samples, which more closely reflects the practical cost of generating fake interactions. This bounded fake-data perspective not only better reflects settings such as review manipulation, but also makes how the attacker allocates its limited effort across arms a central question.

More importantly, existing attacks on stochastic bandits typically steer the learner toward a target arm by suppressing the observed rewards of non-target arms~\citep{jun2018adversarial,zuo2023adversarial,zuo2024near,zeng2025practical}. Although such attacks admit strong theoretical guarantees with controlled cost, their mechanism differs from common real-world manipulation practices, where attackers often directly promote a chosen item through favorable feedback. For example, researchers have observed Amazon sellers recruiting reviewers through private Facebook groups and reimbursing purchases in exchange for five-star reviews, sometimes with additional commissions~\citep{he2022detecting}. Direct promotion is also intuitively attractive: a favorable review benefits the target itself, whereas suppressing a single competitor affects only one of many alternatives. This gap raises a natural question: when is directly promoting the target not only effective, but necessary for a cost-efficient attack?

Motivated by this mismatch, we study bounded offline attacks on warm-start stochastic bandits under two practical constraints. First, every injected reward must lie within the valid feedback range, such as binary clicks or $1$--$5$ star ratings. Second, the attacker can intervene only before deployment, since maintaining access throughout the online learning process may be difficult or impossible. 
Under these constraints, target promotion and non-target suppression play fundamentally different roles. Injecting favorable feedback on the target improves its comparison against all competitors, but this initial advantage is gradually diluted as genuine online feedback accumulates. In contrast, suppressing a competitor requires injecting unfavorable feedback specifically for that arm, and bounded rewards limit how much each injected sample can change its empirical estimate. The attacker must therefore decide not only how many samples to inject, but also how to allocate these injections across arms. This leads to our central question:

\begin{center}
\textit{How should an attacker allocate its power between the target arm and its competitors?}
\end{center}

We answer this question by jointly optimizing target promotion and non-target suppression. We propose a bounded offline attack against warm-start UCB and establish both an upper bound on its cost and a lower bound for any successful attack. Our lower bound analysis further reveals that the optimal allocation is necessarily target-centric: any attack achieving the leading-order optimal cost that makes the target arm selected in nearly all online rounds must spend a nonvanishing fraction of its cost on promoting the target. We further extend our strategy to the class of suppression-bounded algorithms studied by~\citet{zeng2025practical}, with algorithm-specific analyses for Thompson Sampling and $\epsilon$-greedy in Appendices~\ref{app:analysis-TS} and~\ref{app:analysis-EG}. These results demonstrate the broader applicability of our strategy beyond UCB.

\begin{table*}[t]
    \small
    \centering
    \caption{Comparison of bandit attack models and guarantees.}
    \label{tab:bandit_attack_compare_revised}
    \begin{threeparttable}
        \begin{tabular*}{\textwidth}{@{\extracolsep{\fill}} l c c c c c @{}}
            \toprule
            Attack Strategy & Timing & Bounded Rewards & Target Promotion & Allocation Analysis & Opt. \\
            \midrule
            \citet{jun2018adversarial}
                & online
                & \ding{55}
                & \ding{55}
                & \ding{55}
                & \ding{55} \\
            \citet{liu2019data}
                & offline
                & \ding{55}
                & \ding{51}
                & \ding{55}
                & \ding{51} \\
            \citet{xu2021observation}
                & online
                & \ding{51}
                & \ding{51}
                & \ding{55}
                & \ding{55} \\
            \citet{zuo2024near}
                & online
                & \ding{55}
                & \ding{55}
                & \ding{55}
                & \ding{51} \\
            \citet{zeng2025practical}
                & online
                & \ding{51}
                & \ding{55}
                & \ding{55}
                & \ding{55} \\
            \citet{hosseini2026efficient}
                & offline
                & \ding{55}
                & \ding{55}
                & \ding{55}
                & \ding{55} \\
            Ours (Sec.~\ref{sec:attack})
                & offline
                & \ding{51}
                & \ding{51}
                & \ding{51}
                & \ding{51} \\
            \bottomrule
        \end{tabular*}

        \begin{tablenotes}[flushleft]
            \footnotesize
            \item Prior work separately studies offline poisoning, bounded attacks, or attack-cost optimality, but does not characterize how attack effort should be allocated across arms. Our work combines bounded offline injection with direct target promotion and proves that cost-optimal attacks require substantial effort on the target arm itself. The online variant of \citet{liu2019data} is omitted since it does not satisfy any of the properties in the columns.
            \end{tablenotes}
    \end{threeparttable}
\end{table*}

\subsection{Our Contributions}

We summarize our contributions as follows:
\begin{itemize}
    \item We formulate \emph{bounded offline attacks on warm-start stochastic bandits}, measuring attack cost by the number of valid action--reward samples injected before deployment. This model captures settings such as fake-review manipulation, where the cost is associated with generating fake interactions rather than the magnitude of reward modification.

    \item We establish \emph{upper and lower bounds} on the cost of bounded offline attacks against UCB. In the near-boundary regime, these bounds match to leading order, yielding the optimal sublinear cost scale
    $\Theta\!\left(T^{2/3}(K\log T)^{1/3}\right)$ for $K$ arms and $T$ rounds. Notably, we bridge this theory–practice gap by showing that cost-optimal attacks must allocate a nonvanishing fraction of their cost to the target arm.

    \item We extend our attack strategy beyond UCB to Thompson Sampling, $\epsilon$-greedy, and the broader class of \emph{suppression-bounded} algorithms introduced by~\citet{zeng2025practical}. This extension shows that the same combination of target promotion and non-target suppression applies to a broad class of bandit algorithms.

    \item We evaluate our attacks on MovieLens~\citep{harper2015movielens} against existing baselines. The experiments demonstrate lower attack costs while maintaining effective target-arm manipulation, and empirically validate the predicted target-centric cost allocation.
\end{itemize}

\subsection{Related Work}

Research on adversarial manipulation of bandits follows two complementary directions: designing attack algorithms, and developing robust algorithms against adversarial attack.

\textbf{Adversarial Attacks.}
\citet{jun2018adversarial} showed that sublinear-cost attacks can steer UCB and $\epsilon$-greedy toward a target arm, while \citet{liu2019data} developed an adaptive attack without knowledge of the learner's algorithm. \citet{zuo2024near} established near-optimal attacks against UCB and TS with matching lower bounds. Other work studies restricted intervention: \citet{xu2021observation} proposed an observation-free attack during prescribed online phases, and \citet{zeng2025practical} introduced bounded fake data injection with constraints on injection timing and frequency. A common strategy in this line of work is to favor the target by suppressing the observed rewards of non-target arms. Attacks have also been studied in contextual, linear, adversarial, and learning-to-rank bandits~\citep{garcelon2020adversarial,wang2022linear,ma2023adversarial,zuo2023adversarial}. Beyond these works, \citet{liu2019data} formulated offline reward-poisoning attacks as convex optimization problems for several bandit algorithms. \citet{hosseini2026efficient} studied reward-model perturbations before training an offline bandit, without subsequent online interaction.

\textbf{Robust Bandits.}
\citet{lykouris2018stochastic} established corruption-dependent regret bounds for stochastic bandits without prior knowledge of the corruption level, and \citet{gupta2019better} improved these guarantees. \citet{kapoor2019corruption} used robust statistical estimation to tolerate corrupted rewards in multi-armed and linear contextual bandits. For linear bandits, robust phased elimination~\citep{bogunovic2021stochastic} and confidence-weighted regression~\citep{he2022nearly} control the influence of corrupted observations, with guarantees that account for whether the corruption level is known. Robustness has also been studied in Gaussian process~\citep{bogunovic2022robust}, Lipschitz~\citep{kang2023robust}, and combinatorial bandits~\citep{xu2021simple}. The threat models considered in these works are typically weaker than those studied in the attack literature, since the adversary is not allowed to observe the learner’s actions before corrupting the feedback. Our work follows the attack-design direction and adopts the fake data injection model of~\citet{zeng2025practical} in an offline-to-online setting.

\section{Warm-Start Bandit Setup}
\label{sec:model}

We consider a stochastic multi-armed bandit with arm set $\mathcal K \triangleq \{1,\dots,K\}$. Each arm $i \in \mathcal K$ is associated with an unknown reward distribution with mean $\mu_i$. Rewards are supported on $[\rl,\ru]$ and are $\sigma^2$-sub-Gaussian, where $\sigma^2$ is known to the learner.

The learner is initialized with an offline log $\mathcal D^{\off} \triangleq \bigl( (a_s, r_s) \bigr)_{s = 1}^{T_0}$ containing $T_0$ action--reward observations, where $a_s \in \mathcal K$ denotes the arm selected in observation round $s$, and $r_s \in [\rl, \ru]$ denotes the corresponding observed reward. For each arm $i$, let $N_i^{\off} \triangleq \sum_{s=1}^{T_0} \mathbf 1\{a_s=i\}$ denote its offline count. After initialization, the learner interacts with the bandit environment from round $T_0+1$ through round $T$. Thus, the online rounds are $\mathcal T^{\on}\triangleq\{T_0+1,\dots,T\}$, with online horizon $H\triangleq|\mathcal T^{\on}|=T-T_0$.

At the beginning of each online round $t\in\mathcal T^{\on}$, the learner's history consists of the offline log and the preceding online observations:
\[
    \mathcal H_t \triangleq \left(\mathcal D^{\off}, \bigl( (a_s, r_s)\bigr)_{s=T_0+1}^{t-1}\right).
\]
A general warm-start bandit algorithm is specified by a policy sequence $\Pi \triangleq \{\pi_t\}_{t \in \mathcal T^{\on}}$, where $\pi_t (\cdot \mid \mathcal H_t)\in\Delta(\mathcal K)$. In round $t$, the learner samples $a_t\sim\pi_t(\cdot\mid\mathcal H_t)$, observes the corresponding reward $r_t$, and appends $(a_t, r_t)$ to its history. We write $N_i^{\on} (t) \triangleq \sum_{s = T_0+1}^{t - 1} \mathbf 1\{a_s = i\}$ for the number of online pulls of arm $i$ before round $t$.

\textbf{UCB Specialization.} We next instantiate this general framework using the UCB algorithm considered by \citet{jun2018adversarial}, following the prototype in \citet[Section~2.2]{bubeck2012regret}. UCB summarizes the available history $\mathcal H_t$ through the number of observations and the empirical mean associated with each arm. For every arm $i \in \mathcal K$, define
    $N_i(t) \triangleq \sum_{s=1}^{t-1}\mathbf{1}\{a_s = i\}=N_i^{\off}+N_i^{\on}(t)$ and $\hat{\mu}_i(t) \triangleq {\sum_{s = 1}^{t-1}\mathbf{1}\{a_s = i\}r_s} / {N_i(t)}$ whenever $N_i(t)>0$.
The UCB index of arm $i$ is
\begin{equation*}
    u_i(t) \triangleq \begin{cases}
        \hat{\mu}_i(t)+3\sigma\sqrt{\frac{\log t}{N_i(t)}}, & N_i(t)>0, \\
        +\infty, & N_i(t)=0.
    \end{cases}
\end{equation*}
The learner selects $a_t \in \arg\max_{i \in \mathcal K} u_i(t)$.
In the absence of attacks, it achieves sublinear regret.

\section{Offline Attack Model}
\label{sec:problem}

\subsection{Attack Setting and Attack Cost}

Let $\mathcal D^{\rm cl} \triangleq \bigl((a_s,r_s)\bigr)_{s=1}^{T_0^{\rm cl}}$ denote the clean offline log, where $T_0^{\rm cl} = |\mathcal D^{\rm cl}|$ is the number of genuine observations.
For each arm, its clean-log and subsequent online rewards form a prefix of the same i.i.d. reward stream. Arm choices may depend on past observations; each chosen arm's next reward is then observed and retained, with no filtering based on its value.
Before deployment, the attacker observes $\mathcal D^{\rm cl}$ and knows the learner's policy $\Pi$, the terminal time $T$, and the valid reward range $[\rl, \ru]$. The attacker acts only at this stage and cannot modify the learner's actions, observed rewards, or history after online interaction begins.

For each arm $i$, define its clean offline count and empirical mean as
\[
    N_i^{\rm cl}\triangleq\sum_{s=1}^{T_0^{\rm cl}}\mathbf 1\{a_s=i\},
    \qquad
    \hat\mu_i^{\rm cl}\triangleq
    \frac{1}{N_i^{\rm cl}}\sum_{s=1}^{T_0^{\rm cl}}\mathbf 1\{a_s=i\}r_s,
\]
where the empirical mean is defined whenever $N_i^{\rm cl}>0$.
During online interaction, $N_i^{\rm cl}(t)$ and $\hat\mu_i^{\rm cl}(t)$ denote the count and empirical mean of all genuine observations of arm $i$ available before round $t$, including the clean offline observations but excluding injected samples. Thus, at the start of online interaction, $N_i^{\rm cl}(T_0+1)=N_i^{\rm cl}$ and $\hat\mu_i^{\rm cl}(T_0+1)=\hat\mu_i^{\rm cl}$.

A candidate injection plan is represented by $(\mathbf n,\mathbf y)$, where $\mathbf n \triangleq(n_1,\dots,n_K) \in \mathbb Z_{\ge 0}^K$ specifies the number of fake observations injected for each arm and $\mathbf y \triangleq(y_{i, j}: i \in \mathcal K,\ j = 1,\dots, n_i)$ specifies the corresponding fake rewards. Every injected reward must lie in $[\rl,\ru]$ and, for discrete feedback, in the legal feedback alphabet.

Let $\mathcal D^{\mathrm F}(\mathbf n,\mathbf y)$ denote the collection of fake action--reward pairs $\{(i, y_{i,j}): i \in \mathcal K,\ j = 1, \dots, n_i\}$ in the order chosen by the strategy.
The learner is initialized using the attacked warm-start log
\[
    \mathcal D^{\off}(\mathbf n, \mathbf y) \triangleq \mathcal D^{\rm cl} \mathbin{\Vert} \mathcal D^{\mathrm F}(\mathbf n, \mathbf y),
\]
where $\mathbin{\Vert}$ denotes sequence concatenation.

\begin{definition}[Attack Cost]
    The attack cost is the number of injected samples:
    \begin{equation}
        \label{eq:attack-cost}
        \mathcal C(\mathbf n)\triangleq\sum_{i=1}^K n_i.
    \end{equation}
\end{definition}
Consequently, $N_i^{\off}=N_i^{\rm cl}+n_i$ and
\[
    T_0 = \left|\mathcal D^{\off}\right| = T_0^{\rm cl}+\mathcal C(\mathbf n),
    \qquad
    H=T-T_0.
\]
Without loss of generality, we label the \emph{target arm} as $K$.

\begin{definition}[$(\xi, \delta)$-Success Criterion]
    Fix $\delta \in (0,1)$ and $\xi \colon \mathbb N \to \mathbb R_{\ge0}$.
    An offline attack is $(\xi,\delta)$-successful if
    \begin{equation}
        \Pr \left( \sum_{i = 1}^{K - 1} N_i^{\on} (T + 1) \le \xi \right) \ge 1 - \delta
        \quad\text{and}\quad \mathcal C(\mathbf n)=o(T).
        \label{eq:offline-attack-problem}
    \end{equation}
\end{definition}

We summarize the offline injection workflow in~\Cref{alg:attack}; concrete allocation strategies are developed in the next section.

\begin{algorithm}[H]
    \caption{Offline Injection Attack}
    \label{alg:attack}
    \KwIn{$\mathcal D^{\rm cl}$, $K$, $\Pi$, $T$, $[\rl,\ru]$, $\xi$ and $\delta$, Target Arm ($K$ in our analysis)}
    \KwOut{Attacked warm-start log $\mathcal D^{\off}$}
    Determine an injection plan $(\mathbf n,\mathbf y)$ satisfying Eq.~\eqref{eq:offline-attack-problem}\;
    Construct a fake log $\mathcal D^{\mathrm F}(\mathbf n,\mathbf y)$ containing $(i,y_{i,j})$ for every $i\in\mathcal K$ and $j=1,\dots,n_i$\;
    Set $\mathcal D^{\off}\leftarrow\mathcal D^{\rm cl}\mathbin{\Vert}\mathcal D^{\mathrm F}(\mathbf n,\mathbf y)$\;
    \Return $\mathcal D^{\off}$\;
\end{algorithm}

\section{Attack Strategies}
\label{sec:attack}

We first instantiate Algorithm~\ref{alg:attack} for UCB, then extend the strategy to suppression-bounded algorithms from \citet{zeng2025practical}.

\subsection{Attack on UCB}
\label{sec:analysis-ucb}
We consider the zero-tolerance case $\xi=0$, where the attacker aims to make UCB select the target arm in every online round. The \emph{high-level idea} is to (1) enlarge the gap between the target arm and the non-target arms, and (2) make the exploration bonus sufficiently small for each non-target arm, so that $u_K(t) > u_i(t)$ throughout the online phase.
To ensure this, we seek to solve
\[
    \min_{\mathbf n, \mathbf y}\quad \mathcal C(\mathbf n)
    \quad \text{subject to}\quad
    u_K(t)>u_i(t),
    \qquad \forall i<K,\ t\in\mathcal T^{\on}.
\]
Let $R \triangleq \ru - \rl$ and assume $N_i^{\rm cl} \ge 1$ for every arm.
Write $b_i(t)=3\sigma\sqrt{\log t/N_i(t)}$.
Our construction keeps the target arm's empirical mean above a threshold $z$ and bounds each non-target arm's UCB index by
$z+b_K(t)$.
Specifically, we enforce
\[
    u_i(t)
    =\hat\mu_i(t)+b_i(t)
    \le z+b_K(t)
    <\hat\mu_K(t)+b_K(t)
    =u_K(t),
    \qquad \forall i<K,\ t\in\mathcal T^{\on}.
\]
For a confidence parameter $\delta\in(0,1/2)$, define
\[
    \beta_\delta(N)
    \triangleq
    \sqrt{\frac{2\sigma^2}{N}
    \log\frac{\pi^2KN^2}{3\delta}},
    \qquad
    \underline\mu_K
    \triangleq
    \max\{\rl,\hat\mu_K^{\rm cl}
    -2\beta_\delta(N_K^{\rm cl})\}.
\]
Set $\varepsilon \triangleq \underline \mu_K - \rl$. We choose
\[
    z_T^\star
    =\rl+\left(
    \frac{18R\sigma^2(K-1)\log T}{T}
    \right)^{1/3},
    \qquad
    z=\rl+\max\left\{
    z_T^\star-\rl,\,
    \min\{\varepsilon,R/2\}
    \right\}.
\]
Write $B_i = N_i^{\rm cl}(\hat \mu_i^{\rm cl} - \rl)$ and $d(z) = z - \rl + 3\sigma \sqrt{\log T / T}$.
We set
\begin{equation}\label{eq:ucb-closed-budget}
    n_i= \left\{
    \begin{aligned}
        &\left\lceil\frac{T(z-\underline\mu_K)_+}{\ru-z}\right\rceil+1, \quad\text{for } i = K; \\
        &\left\lceil\left[
        {\left(\frac{3 \sigma \sqrt{\log T} + \sqrt{9 \sigma^2 \log T + 4 d(z) B_i}}{2d(z)}\right)}^2 - N_i^{\rm cl}
        \right]_+\right\rceil,\quad \text{for } i < K.
    \end{aligned}
    \right.
\end{equation}
Using these counts, construct the fake log in~\Cref{alg:attack} as
\[
\mathcal D^{\mathrm F}
=
\left(
\Big\Vert_{i=1}^{K-1}
\bigl((i,\rl)\bigr)_{j=1}^{n_i}
\right)
\Big\Vert
\bigl((K,\ru)\bigr)_{j=1}^{n_K}.
\]
We analyze this construction under the following scaling assumptions.
\begin{assumption}\label[assumption]{asp:scaling-assumptions}
$R = \Theta(1)$, $\sigma = \Theta(1)$, and $N_i^{\rm cl} = \Theta(1)$ uniformly over $i \in \mathcal K$.
\end{assumption}
In particular, $T_0^{\rm cl} = \Theta(K) = o(T)$, so every attack satisfying $\mathcal C(\mathbf n) = o(T)$ has $H=T(1-o(1))$.
Define $S_T\triangleq T^{2/3}(K\log T)^{1/3}$. 
We now give the theoretical guarantee of our attack strategy.
\begin{theorem}\label{thm:attack-cost-ucb}
Fix $\delta\in(0,1/2)$. Under~\Cref{asp:scaling-assumptions}, for all sufficiently large $T$, the above attack satisfies
\[
    \Pr\left(\sum_{i=1}^{K-1} N_i^\on (T+1) = 0\right) \ge 1 - \delta
    \quad \text{and} \quad
    \mathcal C(\mathbf n) = \mathcal O \left(
        \frac{K \log T}{(\varepsilon \vee (S_T / T))^2}
    \right).
\]
\end{theorem}

\begin{proof}[Proof Sketch]
If $\varepsilon\le z_T^\star-\rl$, then $z=z_T^\star$; substituting into Eq.~\eqref{eq:ucb-closed-budget} gives $n_K=O(S_T)$ and $\sum_{i<K}n_i=O(S_T)$. If $\varepsilon>z_T^\star-\rl$, then $z\le\underline\mu_K$ and $z-\rl=\Theta(\varepsilon)$, giving $n_K=1$ and $\sum_{i<K}n_i=O(K\log T/\varepsilon^2)$. Combining the two cases yields the stated cost bound.
\end{proof}

\begin{theorem}
\label{thm:attack-cost-ucb-lower}
\label[theorem]{cor:ucb-gap-cost}
Under~\Cref{asp:scaling-assumptions}, fix $\delta\in(0,1/2)$. Let $\Delta_K \triangleq \mu_K - \rl$. For all sufficiently large $T$, every $(0,\delta)$-successful bounded offline attack on UCB satisfies
\[
    \mathcal C(\mathbf n)=\Omega\left(
        \frac{K\log T}{{(\Delta_K\vee(S_T/T))}^2}
    \right)
\]
with probability at least $1-\delta-1/T$.
\end{theorem}

\begin{proof}[Proof Sketch]
Consider the final round $T$, when UCB selects the target on the zero-tolerance success event. By then, target injections have been diluted by $\Theta (T)$ genuine observations. Writing $u_K(T) = \rl+\rho$, we have $\rho \le \Delta_K + O(n_K/T + \sqrt{\log T / T})$ with high probability. Since the target is selected at $T$, every competitor satisfies $u_i(T) \le u_K(T)$. Its empirical mean is at least $\rl$, so its exploration bonus cannot exceed $\rho$, forcing $N_i^{\rm cl} + n_i\ge 9 \sigma^2 \log T / \rho^2$. Combining these constraints on target and non-target injection costs and minimizing over $\rho$ yields the cost lower bound.
\end{proof}

The upper and lower bounds depend on $\varepsilon$ and $\Delta_K$, respectively. Their gap-dependent rates match when $\varepsilon = \Omega (\Delta_K)$.
When $\varepsilon = \Omega (1)$, our upper and lower bounds match at $\Theta(K \log T)$. This cost order also agrees with the attack costs of~\citet{zeng2025practical}. 

Our bounds identify $S_T/T$ as the \emph{boundary scale} for the target arm's gap from the lower reward bound $\rl$. We refer to the regime $\Delta_K = o(S_T/T)$ as the \emph{near-boundary regime}. 
In this regime, we further sharpen the bounds to match in their leading constants, thereby establishing the asymptotic optimality of our construction. The optimal leading cost is obtained by minimizing the sum of the leading costs of target promotion and non-target suppression. This minimization yields the closed-form value
\[
    \Lambda_T \triangleq 3 {\left(\frac{9 \sigma^2(K-1) T^2\log T}{4 R^2}\right)}^{1/3} = \Theta(S_T).
\]

\begin{corollary}\label[corollary]{cor:ucb-near-boundary-optimality}
Under~\Cref{asp:scaling-assumptions}, fix $\delta\in(0,1/2)$ and assume $\Delta_K=o(\Lambda_T/T)$. For all sufficiently large $T$, our construction has $\mathcal C(\mathbf n) \le (1+o(1))\Lambda_T$. And every $(0,\delta)$-successful bounded offline attack on UCB satisfies
\[
    \Pr\left[\mathcal C(\mathbf n)\ge(1-o(1))\Lambda_T\right]
    \ge1-\delta-\frac1T.
\]
\end{corollary}

We next characterize the cost allocated to the target arm and show that this allocation is necessary for asymptotic optimality.

\begin{lemma}\label[lemma]{lem:ucb-allocation-necessity}
Under~\Cref{asp:scaling-assumptions}, fix $\delta\in(0,1/2)$. For $\varepsilon=o(\Lambda_T/T)$, our construction satisfies $n_K = \left(\frac{2}{3} + o(1)\right) \Lambda_T$.
If $\Delta_K = o(\Lambda_T/T)$, then for all sufficiently large $T$, every $(o(\Lambda_T), \delta)$-successful bounded offline attack on UCB with $\mathcal C(\mathbf n) \le (1+o(1)) \Lambda_T$ satisfies
\[
    \Pr\!\left[
        n_K = \left(\frac{2}{3} + o(1)\right) \Lambda_T
    \right]\ge1-\delta-\frac1T.
\]
\end{lemma}

\begin{remark}
The $2:1$ cost split balances target promotion against non-target suppression. In the near-boundary regime, raising the target's empirical mean has a leading cost linear in its elevation above the lower reward bound. UCB's $N_i^{-1/2}$ exploration bonuses make each competitor's required injection count, and hence their sum, decrease as the inverse square of that elevation. Minimizing the sum of these leading cost terms yields the $2:1$ split between target and aggregate non-target costs. Appendix~\ref{app:allocation} compares this mechanism across algorithms.
\end{remark}

For larger tolerances $\xi = O(\Lambda_T)$, \Cref{lem:ucb-target-budget} establishes that near-boundary successful attacks with a deterministic $O(\Lambda_T)$ cost bound must still incur a nonvanishing fraction of their total cost on the target arm.
Proofs of~\Cref{thm:attack-cost-ucb,thm:attack-cost-ucb-lower} appear in Appendices~\ref{app:proof-ucb-upper-bound} and~\ref{app:proof-ucb-lower-bound}, respectively; those of~\Cref{cor:ucb-near-boundary-optimality} and~\Cref{lem:ucb-allocation-necessity} appear in Appendix~\ref{app:ucb-optimality}.

In the near-boundary regime, our construction is \emph{target-centric}: for any fixed number of competitors, injections on the target account for most of the total cost with high probability, as required for leading-order optimality. This emphasis on favorable feedback for the target qualitatively reflects real-world manipulation practices, such as sellers purchasing positive reviews for their own products~\citep{he2022detecting}.

\textbf{Comparison to \citet{xu2021observation}.}
In the near-boundary regime, the prescribed cost of \citet{xu2021observation}'s UCB-based attack scales as $\Theta (K \log T / \Delta_K^2)$. Consequently, the cost is $\Omega(S_T)$ when $\Delta_K = \mathcal O (S_T/T)$, ceases to be sublinear when $\Delta_K = \mathcal O(\sqrt{K \log T / T})$, and is undefined at $\Delta_K = 0$. In contrast, our attack remains applicable and covers this boundary case.

\subsection{Extension to Suppression-Bounded Algorithms}

We extend the strategy to algorithms whose additional pulls of persistently dominated arms admit a high-probability bound.

\begin{definition}[Suppression-Bounded Algorithms from~\citet{zeng2025practical}]
\label[definition]{def:suppression-bounded-alg}
Fix a horizon $T$ and a confidence level $\delta \in (0, 1)$. A bandit algorithm $\mathcal A$ is said to be \emph{suppression-bounded} if, for any arm $i$, there exists a deterministic bound $B_i^{\mathcal A}(t, T, \delta)$ such that the following implication holds with probability at least $1 - \delta$: if \[
    \hat \mu_i(s) < \max_{j \in \mathcal K} \hat \mu_j (s), \qquad \text{for all } s \in \{t + 1, \ldots, T\},
\]
then \[
    N_i(T) - N_i(t) \le B_i^{\mathcal A}(t, T, \delta).
\]
The quantity $B_i^{\mathcal A}(t, T, \delta)$ is referred to as the \emph{post-suppression pull bound} for arm $i$. This is a conservative upper bound and $B_i^{\mathcal A}(t, T, \delta) \ge B_i^{\mathcal A}(t, T', \delta)$ for all $T' < T$.
\end{definition}

\begin{theorem}[Generic Suppression-Bounded Attack]
\label{thm:pull-count-upper}
Suppose $\mathcal A$ is suppression-bounded and write $B_i^{\mathcal A} \triangleq B_i^{\mathcal A} (T_0^{\rm cl}, T, \delta/K)$. Assume $N_i^{\rm cl} \ge 1$ for every arm. There exists a bounded offline attack that satisfies:
\[
    \Pr \left(\sum_{i=1}^{K-1} N_i^{\on} (T + 1) \le \sum_{i = 1}^{K-1} B_i^{\mathcal A}\right) \ge 1 - 2 \delta 
\] and \[
    \mathcal C(\mathbf n)
    \le\mathcal O\left(
    \inf_{\rho \in (0, R)} \left\{
        K+T\frac{(\rho-\varepsilon)_+}{R-\rho}
        +\frac{R-\rho}{\rho}
            \sum_{i=1}^{K-1}\left(N_i^{\rm cl}+B_i^{\mathcal A}\right)
            \right\}
    \right).
\]
\end{theorem}

The construction raises the target's empirical mean above $\rl+\rho$ and keeps each non-target mean at or below this threshold. The target injections counteract dilution by genuine feedback. The non-target injections offset the clean observations and up to $B_i^{\mathcal A}$ online observations, even if each has reward $\ru$. Each such observation requires approximately $(R-\rho)/\rho$ compensating injections at $\rl$.

Under~\Cref{asp:scaling-assumptions}, if $\sum_{i < K}B_i^{\mathcal A} = o(T)$, let $A_T \triangleq \sum_{i < K}(N_i^{\rm cl} + B_i^{\mathcal A}) = o(T)$. Choosing $\rho = R \sqrt{A_T / T} / (1 + \sqrt{A_T / T})$ and using $\varepsilon \ge 0$ gives
\[
    \mathcal C(\mathbf n) = O \left(K + \sqrt{T A_T} \right)=o(T).
\]
Consequently, $H=T(1-o(1))$ and $N_K^{\on}(T+1)/H=1-o(1)$ with probability at least $1-2 \delta$, yielding a successful attack. 
UCB, TS, and $\epsilon$-greedy all fall into the scope of suppression-bounded algorithms~\citep{zeng2025practical}.

\section{Numerical Illustration}
\label{sec:exp}

\definecolor{myblue}{RGB}{31,119,180}
\definecolor{myorange}{RGB}{255,127,14}
\definecolor{mygreen}{RGB}{44,160,44}
\definecolor{mypurple}{RGB}{148,103,189}
\definecolor{myred}{RGB}{214,39,40}
\definecolor{mybrown}{RGB}{140,86,75}
\colorlet{exptotal}{myblue}
\colorlet{expheuristic}{myred}
\colorlet{exponline}{mypurple}
\colorlet{expdirect}{mybrown}
\colorlet{expxu}{myorange}
\colorlet{exptarget}{myred}
\colorlet{expnon_target}{mypurple}

\begin{figure}[t]
\centering
\begin{minipage}[t]{0.49\linewidth}
\centering
\begin{tikzpicture}
\begin{axis}[
    width=0.97\linewidth, height=4.1cm,
    xlabel={Horizon $T$}, ylabel={Attack Cost},
    label style={font=\small}, tick label style={font=\scriptsize},
    tick align=outside, tick pos=left, axis on top,
    grid=none,
    xmode=normal, ymode=normal,ymin=0,xmin=100000, xmax=1000000, xtick={100000,400000,700000,1000000}, scaled x ticks=base 10:-6,
]

\addplot[name path=UCBhorizonheuristiccostup, draw=none, forget plot] coordinates {(100000,81948.436004) (250000,185782.194635) (400000,270880.281536) (550000,341229.185428) (700000,396862.989292) (850000,445229.117888) (1000000,482408.895286)};
\addplot[name path=UCBhorizonheuristiccostlo, draw=none, forget plot] coordinates {(100000,80719.723996) (250000,181721.645365) (400000,264061.518464) (550000,331364.214572) (700000,388558.610708) (850000,433678.242112) (1000000,472222.224714)};
\addplot[expheuristic, fill opacity=0.12, draw=none, forget plot] fill between[of=UCBhorizonheuristiccostup and UCBhorizonheuristiccostlo];
\addplot[color=expheuristic, thick, solid, mark=square, mark size=2.4pt, mark options={solid, fill=none}] coordinates {(100000,81334.08) (250000,183751.92) (400000,267470.9) (550000,336296.7) (700000,392710.8) (850000,439453.68) (1000000,477315.56)};

\addplot[name path=UCBhorizontotalcostup, draw=none, forget plot] coordinates {(100000,26256.7593119) (250000,48742.7532049) (400000,67027.366405) (550000,83206.0921907) (700000,98026.4112786) (850000,111866.998345) (1000000,124953.144326)};
\addplot[name path=UCBhorizontotalcostlo, draw=none, forget plot] coordinates {(100000,26178.0006881) (250000,48638.5267951) (400000,66907.993595) (550000,83075.1478093) (700000,97886.2687214) (850000,111719.001655) (1000000,124797.335674)};
\addplot[exptotal, fill opacity=0.12, draw=none, forget plot] fill between[of=UCBhorizontotalcostup and UCBhorizontotalcostlo];
\addplot[color=exptotal, thick, solid, mark=o, mark size=2.4pt, mark options={solid, fill=none}] coordinates {(100000,26217.38) (250000,48690.64) (400000,66967.68) (550000,83140.62) (700000,97956.34) (850000,111793) (1000000,124875.24)};

\end{axis}
\end{tikzpicture}
\par\vspace{2pt}{\small (a) Attack Cost vs. Horizon $T$}
\end{minipage}\hfill
\begin{minipage}[t]{0.49\linewidth}
\centering
\begin{tikzpicture}
\begin{axis}[
    width=0.97\linewidth, height=4.1cm,
    xlabel={Number of Arms $K$}, ylabel={Attack Cost},
    label style={font=\small}, tick label style={font=\scriptsize},
    tick align=outside, tick pos=left, axis on top,
    grid=none,
    xmode=normal, ymode=normal,ymin=0,xmin=5, xmax=50, xtick={5,10,20,30,40,50},
]

\addplot[name path=UCBarmsheuristiccostup, draw=none, forget plot] coordinates {(5,241493.166532) (10,482408.895286) (15,656083.893878) (20,767305.663724) (30,877169.010227) (40,924675.205615) (50,946999.257135)};
\addplot[name path=UCBarmsheuristiccostlo, draw=none, forget plot] coordinates {(5,236122.793468) (10,472222.224714) (15,643598.306122) (20,755325.456276) (30,868423.109773) (40,918188.514385) (50,941528.782865)};
\addplot[expheuristic, fill opacity=0.12, draw=none, forget plot] fill between[of=UCBarmsheuristiccostup and UCBarmsheuristiccostlo];
\addplot[color=expheuristic, thick, solid, mark=square, mark size=2.4pt, mark options={solid, fill=none}] coordinates {(5,238807.98) (10,477315.56) (15,649841.1) (20,761315.56) (30,872796.06) (40,921431.86) (50,944264.02)};

\addplot[name path=UCBarmstotalcostup, draw=none, forget plot] coordinates {(5,92660.6901853) (10,124953.144326) (15,147096.667062) (20,164731.499442) (30,192769.414015) (40,215259.626559) (50,234569.324269)};
\addplot[name path=UCBarmstotalcostlo, draw=none, forget plot] coordinates {(5,92540.4298147) (10,124797.335674) (15,146941.172938) (20,164573.820558) (30,192565.425985) (40,215061.053441) (50,234345.195731)};
\addplot[exptotal, fill opacity=0.12, draw=none, forget plot] fill between[of=UCBarmstotalcostup and UCBarmstotalcostlo];
\addplot[color=exptotal, thick, solid, mark=o, mark size=2.4pt, mark options={solid, fill=none}] coordinates {(5,92600.56) (10,124875.24) (15,147018.92) (20,164652.66) (30,192667.42) (40,215160.34) (50,234457.26)};

\end{axis}
\end{tikzpicture}
\par\vspace{2pt}{\small (b) Attack Cost vs. Number of Arms $K$}
\end{minipage}

\par\vspace{2pt}
{\scriptsize\begin{tabular}{cc}
\tikz[baseline=-0.6ex]{\draw[expheuristic, thick, solid] (0,0)--(.48,0); \draw plot[mark=square, mark size=2.4pt, mark options={solid, fill=none, draw=expheuristic}] coordinates {(.24,0)};}\;Clipped Suppression & \tikz[baseline=-0.6ex]{\draw[exptotal, thick, solid] (0,0)--(.48,0); \draw plot[mark=o, mark size=2.4pt, mark options={solid, fill=none, draw=exptotal}] coordinates {(.24,0)};}\;Total Attack Cost
\end{tabular}}

\caption{Attack Cost for UCB.}
\label{fig:ucb_cost_experiments}
\end{figure}
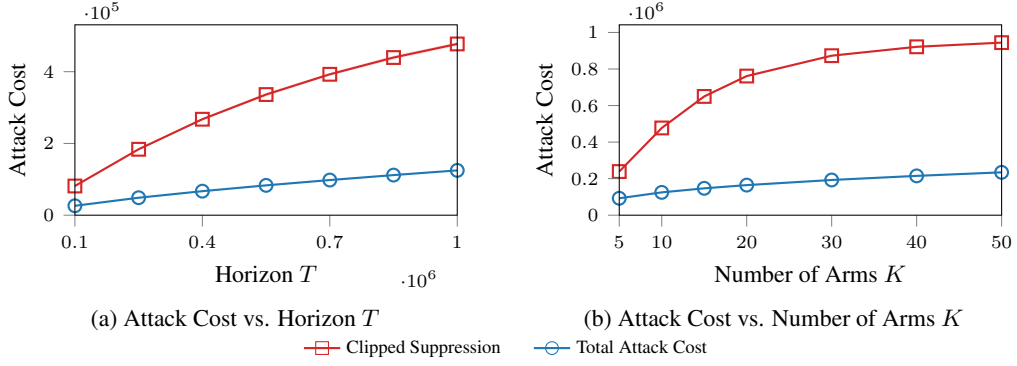
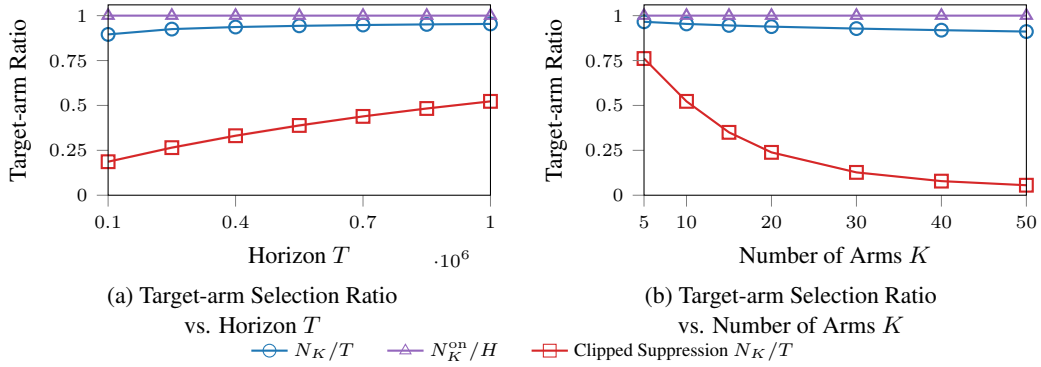
\begin{figure}[t]
\centering
\begin{minipage}[t]{0.49\linewidth}
\centering
\begin{tikzpicture}
\begin{axis}[
    width=0.97\linewidth, height=4.1cm,
    xlabel={Horizon $T$}, ylabel={Target-arm Ratio},
    label style={font=\small}, tick label style={font=\scriptsize},
    tick align=outside, tick pos=left, axis on top,
    grid=none,
    xmode=normal, ymode=normal,ymin=0, ymax=1.06, ytick={0,0.25,0.5,0.75,1},xmin=100000, xmax=1000000, xtick={100000,400000,700000,1000000}, scaled x ticks=base 10:-6,
]

\addplot[name path=UCBhorizontotaltarget_ratioup, draw=none, forget plot] coordinates {(100000,0.895745388008) (250000,0.925182419096) (400000,0.936642240208) (550000,0.94331893884) (700000,0.94787020672) (850000,0.951249090222) (1000000,0.953897266116)};
\addplot[name path=UCBhorizontotaltarget_ratiolo, draw=none, forget plot] coordinates {(100000,0.895197411992) (250000,0.924895020904) (400000,0.936438159792) (550000,0.943155606614) (700000,0.947733678994) (850000,0.951130486248) (1000000,0.953791053884)};
\addplot[exptotal, fill opacity=0.12, draw=none, forget plot] fill between[of=UCBhorizontotaltarget_ratioup and UCBhorizontotaltarget_ratiolo];
\addplot[color=exptotal, thick, solid, mark=o, mark size=2.4pt, mark options={solid, fill=none}] coordinates {(100000,0.8954714) (250000,0.92503872) (400000,0.9365402) (550000,0.943237272727) (700000,0.947801942857) (850000,0.951189788235) (1000000,0.95384416)};

\addplot[name path=UCBhorizononlineonline_ratioup, draw=none, forget plot] coordinates {(100000,1) (250000,1) (400000,1) (550000,1) (700000,1) (850000,1) (1000000,1)};
\addplot[name path=UCBhorizononlineonline_ratiolo, draw=none, forget plot] coordinates {(100000,1) (250000,1) (400000,1) (550000,1) (700000,1) (850000,1) (1000000,1)};
\addplot[exponline, fill opacity=0.12, draw=none, forget plot] fill between[of=UCBhorizononlineonline_ratioup and UCBhorizononlineonline_ratiolo];
\addplot[color=exponline, thick, solid, mark=triangle, mark size=2.4pt, mark options={solid, fill=none}] coordinates {(100000,1) (250000,1) (400000,1) (550000,1) (700000,1) (850000,1) (1000000,1)};

\addplot[name path=UCBhorizonheuristictarget_ratioup, draw=none, forget plot] coordinates {(100000,0.19280276004) (250000,0.273113418542) (400000,0.33984620384) (550000,0.39751960987) (700000,0.444916270417) (850000,0.489790303397) (1000000,0.527777775286)};
\addplot[name path=UCBhorizonheuristictarget_ratiolo, draw=none, forget plot] coordinates {(100000,0.18051563996) (250000,0.256871221458) (400000,0.32279929616) (550000,0.379583299221) (700000,0.43305287244) (850000,0.476201037779) (1000000,0.517591104714)};
\addplot[expheuristic, fill opacity=0.12, draw=none, forget plot] fill between[of=UCBhorizonheuristictarget_ratioup and UCBhorizonheuristictarget_ratiolo];
\addplot[color=expheuristic, thick, solid, mark=square, mark size=2.4pt, mark options={solid, fill=none}] coordinates {(100000,0.1866592) (250000,0.26499232) (400000,0.33132275) (550000,0.388551454545) (700000,0.438984571429) (850000,0.482995670588) (1000000,0.52268444)};

\end{axis}
\end{tikzpicture}
\par\vspace{2pt}{\small (a) Target-arm Selection Ratio\\vs. Horizon $T$}
\end{minipage}\hfill
\begin{minipage}[t]{0.49\linewidth}
\centering
\begin{tikzpicture}
\begin{axis}[
    width=0.97\linewidth, height=4.1cm,
    xlabel={Number of Arms $K$}, ylabel={Target-arm Ratio},
    label style={font=\small}, tick label style={font=\scriptsize},
    tick align=outside, tick pos=left, axis on top,
    grid=none,
    xmode=normal, ymode=normal,ymin=0, ymax=1.06, ytick={0,0.25,0.5,0.75,1},xmin=5, xmax=50, xtick={5,10,20,30,40,50},
]

\addplot[name path=UCBarmstotaltarget_ratioup, draw=none, forget plot] coordinates {(5,0.965773879095) (10,0.953897266116) (15,0.945564097282) (20,0.938826799144) (30,0.927987522111) (40,0.919143516216) (50,0.911437247472)};
\addplot[name path=UCBarmstotaltarget_ratiolo, draw=none, forget plot] coordinates {(5,0.965692400905) (10,0.953791053884) (15,0.945457102718) (20,0.938717840856) (30,0.927847517889) (40,0.919003923784) (50,0.911281232528)};
\addplot[exptotal, fill opacity=0.12, draw=none, forget plot] fill between[of=UCBarmstotaltarget_ratioup and UCBarmstotaltarget_ratiolo];
\addplot[color=exptotal, thick, solid, mark=o, mark size=2.4pt, mark options={solid, fill=none}] coordinates {(5,0.96573314) (10,0.95384416) (15,0.9455106) (20,0.93877232) (30,0.92791752) (40,0.91907372) (50,0.91135924)};

\addplot[name path=UCBarmsonlineonline_ratioup, draw=none, forget plot] coordinates {(5,1) (10,1) (15,1) (20,1) (30,1) (40,1) (50,1)};
\addplot[name path=UCBarmsonlineonline_ratiolo, draw=none, forget plot] coordinates {(5,1) (10,1) (15,1) (20,1) (30,1) (40,1) (50,1)};
\addplot[exponline, fill opacity=0.12, draw=none, forget plot] fill between[of=UCBarmsonlineonline_ratioup and UCBarmsonlineonline_ratiolo];
\addplot[color=exponline, thick, solid, mark=triangle, mark size=2.4pt, mark options={solid, fill=none}] coordinates {(5,1) (10,1) (15,1) (20,1) (30,1) (40,1) (50,1)};

\addplot[name path=UCBarmsheuristictarget_ratioup, draw=none, forget plot] coordinates {(5,0.763877206532) (10,0.527777775286) (15,0.356401693878) (20,0.244674543724) (30,0.131576890227) (40,0.0818114856146) (50,0.0584712171347)};
\addplot[name path=UCBarmsheuristictarget_ratiolo, draw=none, forget plot] coordinates {(5,0.758506833468) (10,0.517591104714) (15,0.343916106122) (20,0.232694336276) (30,0.122830989773) (40,0.0753247943854) (50,0.0530007428653)};
\addplot[expheuristic, fill opacity=0.12, draw=none, forget plot] fill between[of=UCBarmsheuristictarget_ratioup and UCBarmsheuristictarget_ratiolo];
\addplot[color=expheuristic, thick, solid, mark=square, mark size=2.4pt, mark options={solid, fill=none}] coordinates {(5,0.76119202) (10,0.52268444) (15,0.3501589) (20,0.23868444) (30,0.12720394) (40,0.07856814) (50,0.05573598)};

\end{axis}
\end{tikzpicture}
\par\vspace{2pt}{\small (b) Target-arm Selection Ratio\\vs. Number of Arms $K$}
\end{minipage}

\par\vspace{2pt}
{\scriptsize\begin{tabular}{ccc}
\tikz[baseline=-0.6ex]{\draw[exptotal, thick, solid] (0,0)--(.48,0); \draw plot[mark=o, mark size=2.4pt, mark options={solid, fill=none, draw=exptotal}] coordinates {(.24,0)};}\;$N_K/T$ & \tikz[baseline=-0.6ex]{\draw[exponline, thick, solid] (0,0)--(.48,0); \draw plot[mark=triangle, mark size=2.4pt, mark options={solid, fill=none, draw=exponline}] coordinates {(.24,0)};}\;$N_K^{\on}/H$ & \tikz[baseline=-0.6ex]{\draw[expheuristic, thick, solid] (0,0)--(.48,0); \draw plot[mark=square, mark size=2.4pt, mark options={solid, fill=none, draw=expheuristic}] coordinates {(.24,0)};}\;Clipped Suppression $N_K/T$
\end{tabular}}

\caption{Target-arm Selection Ratio for UCB.}
\label{fig:ucb_ratio_experiments}
\end{figure}
\begin{figure}[t]
\centering
\begin{minipage}[t]{0.49\linewidth}
\centering
\begin{tikzpicture}
\begin{axis}[
    width=0.97\linewidth, height=4.1cm,
    xlabel={Horizon $T$}, ylabel={$\mathcal C/\Lambda_T$},
    label style={font=\small}, tick label style={font=\scriptsize},
    tick align=outside, tick pos=left, axis on top, grid=none,
    xmode=normal, ymode=normal,
    xmin=100000, xmax=10000000,
    xtick={100000,2500000,5000000,7500000,10000000},
    scaled x ticks=base 10:-6, scaled y ticks=false,
    ymin=0.962119167386, ymax=1.13527525567,
    yticklabel style={/pgf/number format/fixed, /pgf/number format/precision=3},
]
\addplot[expdirect, opacity=0.65, thin, solid, no marks, forget plot] coordinates {(100000,1) (10000000,1)};

\addplot[name path=UCB5normalized_costup, draw=none, forget plot] coordinates {(100000,0.997196754175) (250000,0.989199944555) (1000000,0.982330865937) (2500000,0.98022223145) (10000000,0.979403349574)};
\addplot[name path=UCB5normalized_costlo, draw=none, forget plot] coordinates {(100000,0.993995654514) (250000,0.986991938278) (1000000,0.981055940467) (2500000,0.979330851586) (10000000,0.978876208188)};
\addplot[exptotal, fill opacity=0.12, draw=none, forget plot] fill between[of=UCB5normalized_costup and UCB5normalized_costlo];
\addplot[color=exptotal, thick, solid, mark=o, mark size=2.4pt, mark options={solid, fill=none}] coordinates {(100000,0.995596204344) (250000,0.988095941417) (1000000,0.981693403202) (2500000,0.979776541518) (10000000,0.979139778881)};

\addplot[name path=UCB10normalized_costup, draw=none, forget plot] coordinates {(100000,1.04777884421) (250000,1.02934083745) (1000000,1.01091647898) (2500000,1.00323364902) (10000000,0.996141779436)};
\addplot[name path=UCB10normalized_costlo, draw=none, forget plot] coordinates {(100000,1.04463597274) (250000,1.02713980257) (1000000,1.00965593019) (2500000,1.00234827333) (10000000,0.99562386859)};
\addplot[expheuristic, fill opacity=0.12, draw=none, forget plot] fill between[of=UCB10normalized_costup and UCB10normalized_costlo];
\addplot[color=expheuristic, thick, solid, mark=square, mark size=2.4pt, mark options={solid, fill=none}] coordinates {(100000,1.04620740848) (250000,1.02824032001) (1000000,1.01028620459) (2500000,1.00279096118) (10000000,0.995882824013)};

\addplot[name path=UCB15normalized_costup, draw=none, forget plot] coordinates {(100000,1.07783955325) (250000,1.05270734902) (1000000,1.02709136329) (2500000,1.01599736027) (10000000,1.00519035001)};
\addplot[name path=UCB15normalized_costlo, draw=none, forget plot] coordinates {(100000,1.0751311139) (250000,1.05083551781) (1000000,1.02600563732) (2500000,1.01524655235) (10000000,1.00474840749)};
\addplot[exponline, fill opacity=0.12, draw=none, forget plot] fill between[of=UCB15normalized_costup and UCB15normalized_costlo];
\addplot[color=exponline, thick, solid, mark=triangle, mark size=2.4pt, mark options={solid, fill=none}] coordinates {(100000,1.07648533357) (250000,1.05177143342) (1000000,1.02654850031) (2500000,1.01562195631) (10000000,1.00496937875)};

\addplot[name path=UCB20normalized_costup, draw=none, forget plot] coordinates {(100000,1.10050172382) (250000,1.0701015502) (1000000,1.03890150294) (2500000,1.0252225055) (10000000,1.011605023)};
\addplot[name path=UCB20normalized_costlo, draw=none, forget plot] coordinates {(100000,1.09801167743) (250000,1.06841199293) (1000000,1.03790707972) (2500000,1.02453565571) (10000000,1.01120316369)};
\addplot[expxu, fill opacity=0.12, draw=none, forget plot] fill between[of=UCB20normalized_costup and UCB20normalized_costlo];
\addplot[color=expxu, thick, solid, mark=diamond, mark size=2.4pt, mark options={solid, fill=none}] coordinates {(100000,1.09925670062) (250000,1.06925677156) (1000000,1.03840429133) (2500000,1.0248790806) (10000000,1.01140409334)};

\addplot[name path=UCB25normalized_costup, draw=none, forget plot] coordinates {(100000,1.11851821487) (250000,1.08382905432) (1000000,1.04811651498) (2500000,1.03238562846) (10000000,1.01654948921)};
\addplot[name path=UCB25normalized_costlo, draw=none, forget plot] coordinates {(100000,1.11598652232) (250000,1.08213815989) (1000000,1.04715077866) (2500000,1.03171222594) (10000000,1.01615279512)};
\addplot[expdirect, fill opacity=0.12, draw=none, forget plot] fill between[of=UCB25normalized_costup and UCB25normalized_costlo];
\addplot[color=expdirect, thick, solid, mark=x, mark size=2.4pt, mark options={solid, fill=none}] coordinates {(100000,1.1172523686) (250000,1.08298360711) (1000000,1.04763364682) (2500000,1.0320489272) (10000000,1.01635114216)};

\end{axis}
\end{tikzpicture}
\par\vspace{2pt}
{\scriptsize\setlength{\tabcolsep}{1pt}\begin{tabular}{ccccc}
\tikz[baseline=-0.6ex]{\draw[exptotal, thick, solid] (0,0)--(.30,0); \draw plot[mark=o, mark size=2.4pt, mark options={solid, fill=none, draw=exptotal}] coordinates {(.15,0)};}\;$K=5$ & \tikz[baseline=-0.6ex]{\draw[expheuristic, thick, solid] (0,0)--(.30,0); \draw plot[mark=square, mark size=2.4pt, mark options={solid, fill=none, draw=expheuristic}] coordinates {(.15,0)};}\;$K=10$ & \tikz[baseline=-0.6ex]{\draw[exponline, thick, solid] (0,0)--(.30,0); \draw plot[mark=triangle, mark size=2.4pt, mark options={solid, fill=none, draw=exponline}] coordinates {(.15,0)};}\;$K=15$ & \tikz[baseline=-0.6ex]{\draw[expxu, thick, solid] (0,0)--(.30,0); \draw plot[mark=diamond, mark size=2.4pt, mark options={solid, fill=none, draw=expxu}] coordinates {(.15,0)};}\;$K=20$ & \tikz[baseline=-0.6ex]{\draw[expdirect, thick, solid] (0,0)--(.30,0); \draw plot[mark=x, mark size=2.4pt, mark options={solid, fill=none, draw=expdirect}] coordinates {(.15,0)};}\;$K=25$
\end{tabular}}

\par\vspace{2pt}{\small (a) Normalized Attack Cost vs. Horizon $T$}
\end{minipage}\hfill
\begin{minipage}[t]{0.49\linewidth}
\centering
\begin{tikzpicture}
\begin{axis}[
    width=0.97\linewidth, height=4.1cm,
    xlabel={Horizon $T$}, ylabel={Cost Share},
    label style={font=\small}, tick label style={font=\scriptsize},
    tick align=outside, tick pos=left, axis on top, grid=none,
    xmode=normal, ymode=normal,
    xmin=100000, xmax=10000000,
    xtick={100000,2500000,5000000,7500000,10000000},
    scaled x ticks=base 10:-6, scaled y ticks=false,
    ymin=0.3, ymax=0.7,
    ytick={0.333333333333,0.666666666667}, yticklabels={$1/3$,$2/3$},
    yticklabel style={/pgf/number format/fixed, /pgf/number format/precision=3},
]
\addplot[expdirect, opacity=0.65, thin, solid, no marks, forget plot] coordinates {(100000,0.666666666667) (10000000,0.666666666667)};
\addplot[expdirect, opacity=0.65, thin, solid, no marks, forget plot] coordinates {(100000,0.333333333333) (10000000,0.333333333333)};

\addplot[name path=UCBtargetallocation_fractionsup, draw=none, forget plot] coordinates {(100000,0.603466986465) (250000,0.616366268446) (1000000,0.63093988288) (2500000,0.638066781426) (10000000,0.646092948826)};
\addplot[name path=UCBtargetallocation_fractionslo, draw=none, forget plot] coordinates {(100000,0.602567890935) (250000,0.615711876589) (1000000,0.630549822971) (2500000,0.637787280622) (10000000,0.645926717738)};
\addplot[exptarget, fill opacity=0.12, draw=none, forget plot] fill between[of=UCBtargetallocation_fractionsup and UCBtargetallocation_fractionslo];
\addplot[color=exptarget, thick, solid, mark=triangle, mark size=2.4pt, mark options={solid, fill=none}] coordinates {(100000,0.6030174387) (250000,0.616039072518) (1000000,0.630744852925) (2500000,0.637927031024) (10000000,0.646009833282)};

\addplot[name path=UCBnon_targetallocation_fractionsup, draw=none, forget plot] coordinates {(100000,0.397432109065) (250000,0.384288123411) (1000000,0.369450177029) (2500000,0.362212719378) (10000000,0.354073282262)};
\addplot[name path=UCBnon_targetallocation_fractionslo, draw=none, forget plot] coordinates {(100000,0.396533013535) (250000,0.383633731554) (1000000,0.36906011712) (2500000,0.361933218574) (10000000,0.353907051174)};
\addplot[expnon_target, fill opacity=0.12, draw=none, forget plot] fill between[of=UCBnon_targetallocation_fractionsup and UCBnon_targetallocation_fractionslo];
\addplot[color=expnon_target, thick, solid, mark=square, mark size=2.4pt, mark options={solid, fill=none}] coordinates {(100000,0.3969825613) (250000,0.383960927482) (1000000,0.369255147075) (2500000,0.362072968976) (10000000,0.353990166718)};

\end{axis}
\end{tikzpicture}
\par\vspace{2pt}
{\scriptsize\begin{tabular}{cc}
\tikz[baseline=-0.6ex]{\draw[exptarget, thick, solid] (0,0)--(.48,0); \draw plot[mark=triangle, mark size=2.4pt, mark options={solid, fill=none, draw=exptarget}] coordinates {(.24,0)};}\;Target $n_K/\mathcal C$ & \tikz[baseline=-0.6ex]{\draw[expnon_target, thick, solid] (0,0)--(.48,0); \draw plot[mark=square, mark size=2.4pt, mark options={solid, fill=none, draw=expnon_target}] coordinates {(.24,0)};}\;Non-target $\sum_{i<K}n_i/\mathcal C$
\end{tabular}}

\par\vspace{2pt}{\small (b) Cost Share vs. Horizon $T$}
\end{minipage}

\caption{Near-Boundary Cost and Allocation for UCB.}
\label{fig:ucb-theory-alignment}
\end{figure}
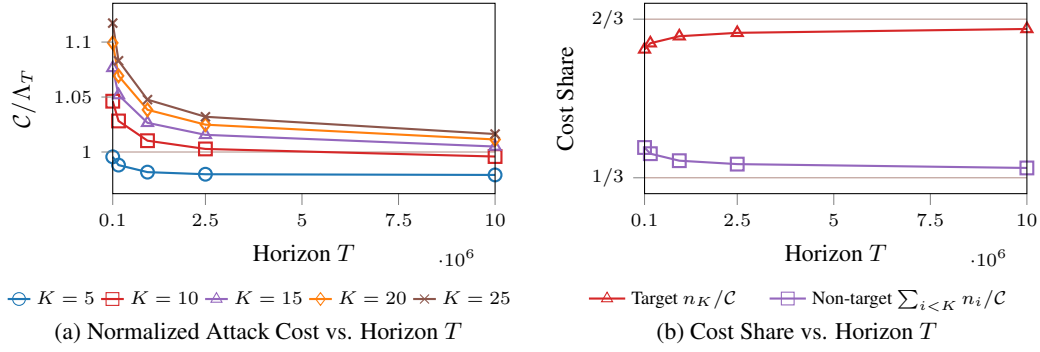
\begin{figure}[t]
\centering
\begin{minipage}[t]{0.49\linewidth}
\centering
\begin{tikzpicture}
\begin{axis}[
    width=0.97\linewidth, height=4.1cm,
    xlabel={$\Delta_K/(S_T/T)$}, ylabel={Attack Cost},
    label style={font=\small}, tick label style={font=\scriptsize},
    tick align=outside, tick pos=left, axis on top, grid=none,
    xmode=normal, ymode=normal, xmin=0, xmax=6,
    xtick={0,1,2,3,4,5,6}, ymin=0,
]

\addplot[name path=ucbxucosttotal0up, draw=none, forget plot] coordinates {(0,95334.9232504) (0.1,89851.0266259) (0.2,84367.010744) (0.3,78883.9037266) (0.4,73399.0934174) (0.5,67912.8243387) (0.75,54197.6742274) (1,40477.0424445) (1.5,19684.4665795) (2,11471.8662465) (3,5344.96008487) (4,3103.89707008) (5,2040.17886024) (6,1445.60063383)};
\addplot[name path=ucbxucosttotal0lo, draw=none, forget plot] coordinates {(0,95134.1967496) (0.1,89651.5333741) (0.2,84167.869256) (0.3,78683.7762734) (0.4,73200.2265826) (0.5,67714.3756613) (0.75,53999.3657726) (1,40279.8775555) (1.5,19535.7734205) (2,11360.2137535) (3,5268.39991513) (4,3048.26292992) (5,1995.30113976) (6,1408.55936617)};
\addplot[exptotal, fill opacity=0.12, draw=none, forget plot] fill between[of=ucbxucosttotal0up and ucbxucosttotal0lo];
\addplot[color=exptotal, thick, solid, mark=o, mark size=2.4pt, mark options={solid, fill=none}] coordinates {(0,95234.56) (0.1,89751.28) (0.2,84267.44) (0.3,78783.84) (0.4,73299.66) (0.5,67813.6) (0.75,54098.52) (1,40378.46) (1.5,19610.12) (2,11416.04) (3,5306.68) (4,3076.08) (5,2017.74) (6,1427.08)};

\addplot[name path=ucbxucostxu0up, draw=none, forget plot] coordinates {(0.3,583449) (0.4,329923) (0.5,212277) (0.75,95618) (1,54519) (1.5,24912) (2,14422) (3,6806) (4,4085) (5,2792) (6,2088)};
\addplot[name path=ucbxucostxu0lo, draw=none, forget plot] coordinates {(0.3,583449) (0.4,329923) (0.5,212277) (0.75,95618) (1,54519) (1.5,24912) (2,14422) (3,6806) (4,4085) (5,2792) (6,2088)};
\addplot[expxu, fill opacity=0.12, draw=none, forget plot] fill between[of=ucbxucostxu0up and ucbxucostxu0lo];
\addplot[color=expxu, thick, solid, mark=diamond, mark size=2.4pt, mark options={solid, fill=none}] coordinates {(0.3,583449) (0.4,329923) (0.5,212277) (0.75,95618) (1,54519) (1.5,24912) (2,14422) (3,6806) (4,4085) (5,2792) (6,2088)};

\end{axis}
\end{tikzpicture}
\par\vspace{2pt}{\small (a) Attack Cost vs. $\Delta_K/(S_T/T)$}

\par\vspace{2pt}
{\scriptsize\begin{tabular}{cc}
\tikz[baseline=-0.6ex]{\draw[exptotal, thick, solid] (0,0)--(.48,0); \draw plot[mark=o, mark size=2.4pt, mark options={solid, fill=none, draw=exptotal}] coordinates {(.24,0)};}\;Ours & \tikz[baseline=-0.6ex]{\draw[expxu, thick, solid] (0,0)--(.48,0); \draw plot[mark=diamond, mark size=2.4pt, mark options={solid, fill=none, draw=expxu}] coordinates {(.24,0)};}\;Xu (2021)
\end{tabular}}

\end{minipage}\hfill
\begin{minipage}[t]{0.49\linewidth}
\centering
\begin{tikzpicture}
\begin{axis}[
    width=0.97\linewidth, height=4.1cm,
    xlabel={$\Delta_K/(S_T/T)$}, ylabel={Target-arm Ratio},
    label style={font=\small}, tick label style={font=\scriptsize},
    tick align=outside, tick pos=left, axis on top, grid=none,
    xmode=normal, ymode=normal, xmin=0, xmax=6,
    xtick={0,1,2,3,4,5,6}, ymin=0.3, ymax=1, ytick={0.3,0.5,0.7,0.9,1},
]

\addplot[name path=ucbxutarget_ratiototal0up, draw=none, forget plot] coordinates {(0,0.965326769821) (0.1,0.965440304852) (0.2,0.965554646026) (0.3,0.96566873876) (0.4,0.965782291765) (0.5,0.965897241881) (0.75,0.966184062323) (1,0.966472569196) (1.5,0.980420226579) (2,0.988595786246) (3,0.994687600085) (4,0.99690773707) (5,0.99796069886) (6,0.998547440634)};
\addplot[name path=ucbxutarget_ratiototal0lo, draw=none, forget plot] coordinates {(0,0.965191390179) (0.1,0.965306095148) (0.2,0.965420313974) (0.3,0.96553342124) (0.4,0.965648228235) (0.5,0.965763438119) (0.75,0.966049977677) (1,0.966339470804) (1.5,0.980271533421) (2,0.988484133754) (3,0.994611039915) (4,0.99685210293) (5,0.99791582114) (6,0.998510399366)};
\addplot[exptotal, fill opacity=0.12, draw=none, forget plot] fill between[of=ucbxutarget_ratiototal0up and ucbxutarget_ratiototal0lo];
\addplot[color=exptotal, thick, solid, mark=o, mark size=2.4pt, mark options={solid, fill=none}] coordinates {(0,0.96525908) (0.1,0.9653732) (0.2,0.96548748) (0.3,0.96560108) (0.4,0.96571526) (0.5,0.96583034) (0.75,0.96611702) (1,0.96640602) (1.5,0.98034588) (2,0.98853996) (3,0.99464932) (4,0.99687992) (5,0.99793826) (6,0.99852892)};

\addplot[name path=ucbxuonline_ratioonline0up, draw=none, forget plot] coordinates {(0,1) (0.1,1) (0.2,1) (0.3,1) (0.4,1) (0.5,1) (0.75,1) (1,1) (1.5,1) (2,1) (3,1) (4,1) (5,1) (6,1)};
\addplot[name path=ucbxuonline_ratioonline0lo, draw=none, forget plot] coordinates {(0,1) (0.1,1) (0.2,1) (0.3,1) (0.4,1) (0.5,1) (0.75,1) (1,1) (1.5,1) (2,1) (3,1) (4,1) (5,1) (6,1)};
\addplot[exponline, fill opacity=0.12, draw=none, forget plot] fill between[of=ucbxuonline_ratioonline0up and ucbxuonline_ratioonline0lo];
\addplot[color=exponline, thick, solid, mark=triangle, mark size=2.4pt, mark options={solid, fill=none}] coordinates {(0,1) (0.1,1) (0.2,1) (0.3,1) (0.4,1) (0.5,1) (0.75,1) (1,1) (1.5,1) (2,1) (3,1) (4,1) (5,1) (6,1)};

\addplot[name path=ucbxutarget_ratioxu0up, draw=none, forget plot] coordinates {(0.3,0.483031) (0.4,0.709201) (0.5,0.81388) (0.75,0.917272) (1,0.953461) (1.5,0.979309) (2,0.988354) (3,0.994816) (4,0.997075) (5,0.998128) (6,0.998695)};
\addplot[name path=ucbxutarget_ratioxu0lo, draw=none, forget plot] coordinates {(0.3,0.483031) (0.4,0.709201) (0.5,0.81388) (0.75,0.917272) (1,0.953461) (1.5,0.979309) (2,0.988354) (3,0.994816) (4,0.997075) (5,0.998128) (6,0.998695)};
\addplot[expxu, fill opacity=0.12, draw=none, forget plot] fill between[of=ucbxutarget_ratioxu0up and ucbxutarget_ratioxu0lo];
\addplot[color=expxu, thick, solid, mark=diamond, mark size=2.4pt, mark options={solid, fill=none}] coordinates {(0.3,0.483031) (0.4,0.709201) (0.5,0.81388) (0.75,0.917272) (1,0.953461) (1.5,0.979309) (2,0.988354) (3,0.994816) (4,0.997075) (5,0.998128) (6,0.998695)};

\end{axis}
\end{tikzpicture}
\par\vspace{2pt}{\small (b) Target-arm Selection Ratio vs. $\Delta_K/(S_T/T)$}

\par\vspace{2pt}
{\scriptsize\begin{tabular}{cc}
\tikz[baseline=-0.6ex]{\draw[exptotal, thick, solid] (0,0)--(.48,0); \draw plot[mark=o, mark size=2.4pt, mark options={solid, fill=none, draw=exptotal}] coordinates {(.24,0)};}\;Ours $N_K/T$ & \tikz[baseline=-0.6ex]{\draw[exponline, thick, solid] (0,0)--(.48,0); \draw plot[mark=triangle, mark size=2.4pt, mark options={solid, fill=none, draw=exponline}] coordinates {(.24,0)};}\;Ours $N_K^{\on}/H$
\end{tabular}}

\par\vspace{2pt}
{\scriptsize\begin{tabular}{c}
\tikz[baseline=-0.6ex]{\draw[expxu, thick, solid] (0,0)--(.48,0); \draw plot[mark=diamond, mark size=2.4pt, mark options={solid, fill=none, draw=expxu}] coordinates {(.24,0)};}\;Xu (2021) $N_K/T$
\end{tabular}}

\end{minipage}
\caption{Comparison with \citet{xu2021observation} for UCB.}
\label{fig:ucb-xu-gap}
\end{figure}
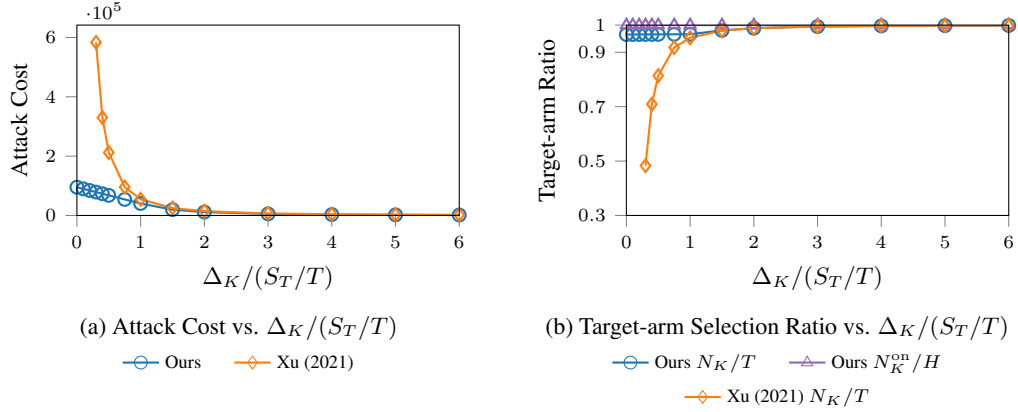

We evaluate our attacks on MovieLens-25M~\citep{harper2015movielens}, mapping ratings of at least four to one and the rest to zero. Each movie defines a Bernoulli arm with its empirical positive-rating rate as the mean; clean and online rewards are sampled independently with replacement. The target, \emph{Glitter} (ID $4775$), has the smallest positive mean, $\mu_K=11/669$; competitors are the most-rated remaining movies. We use $N_i^{\rm cl}=5$, $\sigma=0.5$, and $\delta=0.05$.

\Cref{fig:ucb_cost_experiments} compares our attack strategy with \emph{Clipped Suppression} by varying $T$ at $K=10$ and $K$ at $T=10^6$. Our method consistently incurs a substantially lower attack cost. As $T$ increases, the observed cost growth is consistent with the sublinear bound in \Cref{thm:attack-cost-ucb}. In contrast, \emph{Clipped Suppression} intervenes whenever a non-target arm is pulled, so its attack cost is $T-N_K$ and grows faster with $T$. The cost gap also increases as $K$ grows.

\Cref{fig:ucb_ratio_experiments} further shows that our method achieves lower attack cost while maintaining success. For our method, $N_K^{\on} / H=1$ in every repeat across all tested values of $T$ and $K$, meaning that UCB selects the target arm in every online round after deployment. The overall ratio $N_K / T$, which also includes the warm-start samples, remains close to $1$. In contrast, \emph{Clipped Suppression} yields more non-target pulls. Even at $T=10^6$, the learner selects non-target arms in nearly half of the rounds, and its target-arm selection ratio decreases further as $K$ increases.

We next illustrate the leading-order optimality and target-centric allocation predicted by our theory. As shown in \Cref{fig:ucb-theory-alignment}(a), for each fixed $K$, the normalized attack cost $\mathcal C(\mathbf n)/\Lambda_T$ approaches $1$ as $T$ increases, consistent with \Cref{cor:ucb-near-boundary-optimality}. In \Cref{fig:ucb-theory-alignment}(b), at $K=10$, the target-arm cost share $n_K/\mathcal C(\mathbf n)$ approaches $2/3$, consistent with \Cref{lem:ucb-allocation-necessity}.

We then compare our method with the attack of~\citet{xu2021observation} under the same experimental setting. When $\Delta_K/(S_T/T)<1$, the target lies within the boundary scale $S_T/T$, where attacks become challenging. Our attack maintains sublinear cost in this regime and makes UCB select the target arm in every online round after deployment with high probability. In contrast, the attack cost of \citet{xu2021observation} scales as $O(K\log T / \Delta_K^2)$ and grows rapidly as $\Delta_K$ approaches zero. Hence, their points are omitted when this cost is undefined at $\Delta_K = 0$ or is at least $T$, while their target-arm selection ratio remains below $1$ in the most challenging near-boundary settings.

These experiments demonstrate the effectiveness of our method, with cost shares consistent with the theoretical analysis. We also provide experiments for TS and $\epsilon$-greedy in \Cref{app:ts-experiments} and \Cref{app:eg-experiments}, respectively.

\section{Concluding Remarks}
In the near-boundary regime, target-centric allocation is necessary for successful UCB attacks with tolerance $\xi=O(S_T)$ and a deterministic $O(S_T)$ cost cap: the target must receive a nonvanishing share of the injections on the high-probability event of our lower bound. Our UCB construction attains this cost scale in the near-boundary regime. We also give constructions for TS and $\epsilon$-greedy in the Appendix, while the suppression-bounded analysis covers a broader class of learners.

However, our analysis assumes a learner that uncritically trusts all offline data. A possible direction for future work is to investigate whether discounting, filtering, or auditing offline observations can mitigate the effects of favorable target injections, and how such defenses trade off the benefits of a clean warm start against robustness to corrupted initialization. 
Furthermore, our cost and allocation results assume a small clean offline log. More clean observations make empirical means harder to shift but also reduce exploration uncertainty, potentially changing both the optimal attack cost and its allocation. Characterizing how these competing effects depend on the size and distribution of clean observations across arms is an interesting direction for future work.

\bibliographystyle{plainnat}
\small
\bibliography{bibliography}
\normalsize
\newpage
\appendix
\phantomsection
\addcontentsline{toc}{section}{Appendix}
\begingroup
\small
\tableofcontents
\endgroup

\section{Proofs for UCB}
\label{app:ucb-appendix}
For $\delta\in(0,1)$, recall that
\[
    \beta_\delta(N) = \sqrt{\frac{2\sigma^2}{N} \log \frac{\pi^2KN^2}{3\delta}}.
\]
Define the event
\[
    \mathcal E_\delta\triangleq
    \left\{\forall i\in\mathcal K,\ \forall t \in \mathcal T^{\on},\quad
    \left|\hat\mu_i^{\rm cl}(t) - \mu_i\right|< \beta_\delta (N_i^{\rm cl}(t)) \right\}.
\]
By~\citet[Lemma~1]{jun2018adversarial}, $\Pr(\mathcal E_\delta) \ge 1 - \delta$. The radius is decreasing on the positive integers when $K \ge 2$ and $\delta \in (0,1)$. 
On $\mathcal E_\delta$, we have $0 \le \varepsilon \le \Delta_K$, and
\begin{align*}
    \hat\mu_K^{\rm cl}(t)
    &\ge\max\{\rl,\mu_K-\beta_\delta(N_K^{\rm cl}(t))\}\\
    &\ge\max\{\rl,\hat\mu_K^{\rm cl}-2\beta_\delta(N_K^{\rm cl})\}
    =\underline\mu_K
    \qquad(\forall t\in\mathcal T^{\on}).
\end{align*}

\subsection{Upper Bound Proof}
\label{app:proof-ucb-upper-bound}
Under the event $\mathcal E_\delta$ and by Eq.~\eqref{eq:ucb-closed-budget}, at round $t=T_0+1$, we have
\[
    \hat \mu_K(t)
    \ge
    \frac{n_K\ru + N_K^{\rm cl}(t) \underline\mu_K}
         {n_K + N_K^{\rm cl}(t)}
    \ge
    \frac{n_K \ru + T \underline \mu_K}
         {n_K + T}
    > z.
\]
Since $N_K(t)\le t-1$ and $\log t/(t-1)$ decreases for $t>1$, we have $u_K(t) = \hat \mu_K(t) + 3\sigma\sqrt{\frac{\log t}{N_K(t)}} > z + 3 \sigma \sqrt{\frac{\log T}{T}}$.
Similarly, for every non-target arm $i<K$,
\[
    u_i(t)
    \le
    z + 3\sigma\sqrt{\frac{\log T}{T}}.
\]
Together, these inequalities imply that
\[
    u_K(t)>u_i(t), \qquad \forall i<K,
\]
and hence UCB selects the target arm at round $t$. Repeating the same argument inductively from round $T_0+1$ through round $T$ shows that UCB selects the target arm in every online round.





\begin{proof}[Proof of~\Cref{thm:attack-cost-ucb}]
Write $\rho=z-\rl$, $r_T = S_T/T$, and $\rho_T^\star = z_T^\star - \rl=\Theta(r_T)$. 
We have
\[
    n_K = \frac{T(\rho - \varepsilon)_+}{R - \rho} + O(1)
    \le \frac{T \rho}{R - \rho} + 2.
\]
Since $d(\rl + \rho) = \rho + 3 \sigma \sqrt{\log T/ T} \sim\rho$ and $B_i, N_i^{\rm cl} = O(1)$ uniformly, Eq.~\eqref{eq:ucb-closed-budget} gives
\[
    n_i=(1+o(1))\frac{9\sigma^2\log T}{\rho^2}
    \qquad(i<K).
\]


(1) If $\varepsilon\le\rho_T^\star$, then $\rho=\rho_T^\star$, and the preceding estimates give $n_K = O(Tr_T)$ and $n_i = O(\log T/r_T^2)$ uniformly over $i<K$. Their sum is $O(S_T)$, the asserted scale since $\varepsilon=O(r_T)$.

(2) If $\varepsilon>\rho_T^\star$, then $\rho=\min\{\varepsilon,R/2\}\le\varepsilon$ and $\rho=\Theta(\varepsilon)$. Thus $n_K=1$, while Eq.~\eqref{eq:ucb-closed-budget} gives
\[
    n_i = O\left(1 + B_i/\rho + \sigma^2\log T / \rho^2\right)
    = O(\log T / \varepsilon^2)
    \qquad (i<K).
\]
Combining these two cases yields the result.
%
\end{proof}


\subsection{Lower Bound Proof}
\label{app:proof-ucb-lower-bound}

\begin{lemma}
\label[lemma]{lem:ucb-cost-tradeoff}
Under~\Cref{asp:scaling-assumptions}, there exist constants $c_0,c_1,c_2,c_3>0$ such that the following holds for all sufficiently large $T$. On $\mathcal E_{1/T}$, every bounded offline UCB attack realization with at most $\xi \ge 0$ non-target online pulls and $\mathcal C(\mathbf n) + \xi\le T/2$ admits a $\rho>0$ satisfying
\[
    \begin{aligned}
        n_K &\ge c_1 T (\rho - \Delta_K - c_0 \sqrt{\log T/T})_+,\\
        \sum_{i<K}n_i &\ge \left(c_2\frac{K\log T}{\rho^2}-\xi-c_3K\right)_+.
    \end{aligned}
\]
\end{lemma}

\begin{proof}
Since $H-\xi\ge T/2-T_0^{\rm cl}>0$ for large $T$, the target is selected in at least one online round. Let $t^\star$ be its last such round and set $\rho = u_K(t^\star) - \rl > 0$. We have
\[
    \begin{aligned}
        N_K(t^\star) &= N_K^{\off}+N_K^{\on}(T+1) - 1,\\
        N_K^{\rm cl}(t^\star) &= N_K^{\rm cl} + N_K^{\on}(T+1) - 1.
    \end{aligned}
\]
The bound $N_K^{\on}(T+1)\ge H-\xi$ gives $N_K^{\rm cl}(t^\star)\ge T/3$ for large $T$. Since all rounds after $t^\star$ select non-target arms, we also have $t^\star\ge T-\xi\ge T/2$.

On $\mathcal E_{1/T}$, the genuine target observations have empirical mean at most $\mu_K+\beta_{1/T}(N_K^{\rm cl}(t^\star))$, and each injected target reward is at most $\ru$.
Meanwhile, every non-target empirical mean is at least $\rl$, and its index is at most the selected target's index $\rl+\rho$.
These comparisons give
\begin{equation}\label{eq:ucb-index-constraints}
    \begin{aligned}
        \rho&\le\Delta_K+\frac{(R-\Delta_K)n_K}{N_K(t^\star)}
        +\beta_{1/T}(N_K^{\rm cl}(t^\star))
        +3\sigma\sqrt{\frac{\log t^\star}{N_K(t^\star)}},\\
        &N_i(t^\star) \ge \frac{9\sigma^2\log t^\star}{\rho^2}
        \qquad(i<K).
    \end{aligned}
\end{equation}
The concentration radius and target exploration bonus are both $O(\sqrt{\log T/T})$ for large $T$. Thus the first inequality gives
\[
    \rho\le\Delta_K+\frac{3R n_K}{T}
        +c_0\sqrt{\frac{\log T}{T}}.
\]
Rearranging and using $R=\Theta(1)$ and $n_K\ge0$ yields the target cost constraint.
Write $M=\sum_{i<K}n_i$. Summing the second inequality and using $N_i(t^\star)=N_i^{\rm cl}+n_i+N_i^{\on}(t^\star)$ and $\sum_{i<K}N_i^{\on}(t^\star)\le\xi$ gives
\[
    M+T_0^{\rm cl}+\xi\ge\sum_{i<K}N_i(t^\star)
    \ge\frac{9\sigma^2(K-1)\log t^\star}{\rho^2}.
\]
Using $K-1\ge K/2$, $t^\star\ge T/2$, $T_0^{\rm cl}=O(K)$, and $M\ge0$ yields the non-target cost constraint with uniform constants $c_2,c_3>0$.
\end{proof}

\begin{proof}[Proof of~\Cref{thm:attack-cost-ucb-lower}]
Work on the intersection of the attack's success event with $\mathcal E_{1/T}$, which has probability at least $1-\delta-1/T$ by a union bound.

Recall $r_T=S_T/T$ and set $a_T=(\Delta_K+c_0\sqrt{\log T/T})\vee r_T$, with $c_0$ from~\Cref{lem:ucb-cost-tradeoff}.
On paths with $\mathcal C(\mathbf n)\le T / 2$, add the two constraints in~\Cref{lem:ucb-cost-tradeoff} at $\xi = 0$. Enlarging the offset to $a_T$ only decreases the linear term, giving
\[
    \mathcal C(\mathbf n)+c_3K
    \ge\inf_{\rho>0}\left\{
    c_1T(\rho-a_T)_++\frac{c_2K\log T}{\rho^2}\right\}.
\]
(1) If $\rho\le2a_T$, the inverse-square term is at least $c_2K\log T/(4a_T^2)$.

(2) Otherwise, the linear term is at least $c_1Ta_T\ge c_1K\log T/a_T^2$, since $a_T\ge r_T$ and $Tr_T^3=K\log T$.

Both cases give $\Omega(K\log T/(\Delta_K\vee r_T)^2)$, because $a_T \sim \Delta_K \vee r_T$ by $\sqrt{\log T / T} = o(r_T)$.
\end{proof}

\subsection{Optimality Proofs}\label{app:ucb-optimality}

\begin{proof}[Proof of~\Cref{cor:ucb-near-boundary-optimality}]

\emph{Upper bound.}
Recall $\rho_T^\star=z_T^\star-\rl$. The cost expansions in Appendix~\ref{app:proof-ucb-upper-bound} give, uniformly over clean logs and including rounding,
\[
    \mathcal C(\mathbf n(z_T^\star))
    \le(1+o(1))\left\{
        \frac{T\rho_T^\star}{R}
        +\frac{9\sigma^2(K-1)\log T}{(\rho_T^\star)^2}
    \right\}
    =(1+o(1))\Lambda_T.
\]
(1) If $\varepsilon\le\rho_T^\star$, our construction uses $z=z_T^\star$. 

(2) Otherwise, the target count is one at both thresholds, and raising $z$ only reduces non-target counts. 

\emph{Lower bound.}
Consider a $(0,\delta)$-successful attack on the intersection of its success event with $\mathcal E_{1/T}$, which has probability at least $1-\delta-1/T$.
It suffices to consider paths with $\mathcal C(\mathbf n)\le2\Lambda_T$, since larger costs already satisfy the conclusion. Write $M=\sum_{i<K}n_i$.
Zero tolerance gives $N_K^{\on}(T+1)=H$ and $t^\star=T$. Since $\mathcal C(\mathbf n)+T_0^{\rm cl}=o(T)$, the count identities in the proof of~\Cref{lem:ucb-cost-tradeoff} give $N_K(T),N_K^{\rm cl}(T)=T(1-o(1))$.
Set $\rho=u_K(T)-\rl$. Since $\Delta_K=o(\Lambda_T/T)$ and $\sqrt{\log T/T}=o(\Lambda_T/T)$, the index constraints in Eq.~\eqref{eq:ucb-index-constraints} imply, uniformly over these paths,
\[
    \rho\le\frac{R}{T}\bigl(n_K+o(\Lambda_T)\bigr),
    \qquad
    M+T_0^{\rm cl}\ge\frac{9\sigma^2(K-1)\log T}{\rho^2}.
\]
Eliminating $\rho$ gives
\[
    \bigl(n_K+o(\Lambda_T)\bigr)^2(M+T_0^{\rm cl})
    \ge\frac{9\sigma^2(K-1)T^2\log T}{R^2}
    =\frac{4\Lambda_T^3}{27}.
\]
Since $n_K,M=O(\Lambda_T)$ and $T_0^{\rm cl}=o(\Lambda_T)$, this yields
\begin{equation}\label{eq:ucb-sharp-product}
    n_K^2M\ge(1-o(1))\frac{4\Lambda_T^3}{27}.
\end{equation}
AM--GM applied to $n_K/2,n_K/2,M$ gives
\[
    (1-o(1))\frac{4\Lambda_T^3}{27}
    \le n_K^2M
    \le\frac{4(n_K+M)^3}{27}
    =\frac{4\mathcal C(\mathbf n)^3}{27},
\]
so $\mathcal C(\mathbf n)\ge(1-o(1))\Lambda_T$ on the stated event.
\end{proof}

\begin{proof}[Proof of~\Cref{lem:ucb-allocation-necessity}]
$\varepsilon = o(\Lambda_T/T)$ gives $\varepsilon = o(\rho_T^\star)$ and hence $z = z_T^\star$. Eq.~\eqref{eq:ucb-closed-budget} then gives
\[
    n_K=\frac{T(\rho_T^\star-\varepsilon)}{R-\rho_T^\star}+O(1)
    \sim\frac{T\rho_T^\star}{R}=\frac23\Lambda_T.
\]
This is deterministic under the stated clean-log condition, which follows from $\Delta_K=o(\Lambda_T/T)$ on $\mathcal E_\delta$ because $0\le\varepsilon\le\Delta_K$.

\emph{Necessity.}
Work on the intersection of the attack's success event with $\mathcal E_{1/T}$, which has probability at least $1-\delta-1/T$. Write $M=\sum_{i<K}n_i$, let $t^\star$ be the last online round selecting the target, and set $\rho=u_K(t^\star)-\rl$.
Since $\xi=o(\Lambda_T)$ and $\mathcal C(\mathbf n)\le(1+o(1))\Lambda_T$, we still have $t^\star\ge T-\xi$, $\log t^\star\sim\log T$, and $N_K(t^\star),N_K^{\rm cl}(t^\star)=T(1-o(1))$.
The preceding product argument therefore applies with $M+T_0^{\rm cl}+\xi$ in the non-target constraint. Because $T_0^{\rm cl}+\xi=o(\Lambda_T)$, eliminating $\rho$ again yields Eq.~\eqref{eq:ucb-sharp-product}.
Together with $n_K^2M\le4\mathcal C(\mathbf n)^3/27$ and the cost cap, this implies $\mathcal C(\mathbf n)\sim\Lambda_T$.

Set $x=n_K/\mathcal C(\mathbf n)\in[0,1]$. Dividing the product bound by $\mathcal C(\mathbf n)^3$ gives
\[
    \frac4{27}-o(1)
    \le x^2(1-x)\le\frac4{27}.
\]
The identity
\[
    \frac4{27}-x^2(1-x)
    =\left(x-\frac23\right)^2\left(x+\frac13\right)
    \ge\frac13\left(x-\frac23\right)^2
\]
forces $x=2/3+o(1)$, and hence
\[
    n_K=x\mathcal C(\mathbf n)=\left(\frac23+o(1)\right)\Lambda_T.
\]
\end{proof}

\begin{lemma}[Necessary Target Cost]\label[lemma]{lem:ucb-target-budget}
Under~\Cref{asp:scaling-assumptions}, fix $\delta\in(0,1/2)$ and assume $\Delta_K=o(\Lambda_T/T)$ and $\xi=O(\Lambda_T)$. For all sufficiently large $T$, every $(\xi,\delta)$-successful bounded offline attack on UCB with a deterministic $O(\Lambda_T)$ cost bound satisfies
\[
    n_K=\Omega(\Lambda_T),\qquad
    \frac{n_K}{\mathcal C(\mathbf n)}=\Omega(1)
\]
with probability at least $1-\delta-1/T$.
\end{lemma}
\begin{proof}
Work on the intersection of success with $\mathcal E_{1/T}$, which has probability at least $1-\delta-1/T$. Since $\mathcal C(\mathbf n)+\xi=O(\Lambda_T)=o(T)$, apply~\Cref{lem:ucb-cost-tradeoff} and use its notation.
Since $\sum_{i<K}n_i+\xi+c_3K=O(\Lambda_T)$, the non-target cost constraint gives $\rho=\Omega(\Lambda_T/T)$.
With $\Delta_K,\sqrt{\log T/T}=o(\Lambda_T/T)$, the target cost constraint gives $n_K=\Omega(\Lambda_T)$. The deterministic cost bound then yields $n_K/\mathcal C(\mathbf n)=\Omega(1)$.
\end{proof}

\section{Analysis for Thompson Sampling}\label{app:analysis-TS}

\subsection{Attack Strategy}\label{app:proof-ts-upper-bound}

We consider the Thompson Sampling (TS) algorithm of~\citet{zeng2025practical}, based on the $(\alpha,\psi)$ framework of~\citet{agrawal2017near}. In each online round, it independently samples $\nu_i(t)\sim\mathcal N(\hat\mu_i(t),1/N_i(t))$ for every arm and selects
\[
    a_t=\arg\max_{i\in\mathcal K}\nu_i(t).
\]

For $\delta\in(0,1)$, define
\[
    \gamma(t)\triangleq\sqrt{2\log\frac{\pi^2Kt^2}{3\delta}},
    \qquad
    \mathcal F\triangleq\left\{
        \forall i\in\mathcal K,\ \forall t\in\mathcal T^{\on},\
        |\nu_i(t)-\hat\mu_i(t)|<\frac{\gamma(t)}{\sqrt{N_i(t)}}
    \right\}.
\]
By~\citet[Lemma~B.1]{zeng2025practical}, $\Pr(\mathcal F)\ge1-\delta$.

Fix $z\in(\rl,\ru)$, write $\gamma=\gamma(T)$, and recall $B_i=N_i^{\rm cl}(\hat\mu_i^{\rm cl}-\rl)$.
Inject $\ru$ on the target and $\rl$ on each non-target arm, with counts
\begin{equation}\label{eq:ts-closed-budget}
    n_i=\left\lceil\left[\left(
    \frac{\gamma+\sqrt{\gamma^2+4(z-\rl)B_i}}{2(z-\rl)}
    \right)^2-N_i^{\rm cl}\right]_+\right\rceil,
    \qquad i<K,
\end{equation}
and
\begin{equation}\label{eq:ts-target-constraint}
    n_K = 1 + \min \left\{n \in \mathbb Z_{\ge 0}:\;
    (\ru - z) n \ge(z - \underline \mu_K) m + \gamma \sqrt{n + m},\ \forall m \in[N_K^{\rm cl}, T]
    \right\}.
\end{equation}

\begin{lemma}[TS Threshold Certificate]\label[lemma]{lemma:ts-threshold-attack}
For any $z\in(\rl,\ru)$, the above construction makes TS select the target in every online round on $\mathcal E_\delta\cap\mathcal F$.
\end{lemma}

\begin{proof}
On $\mathcal E_\delta\cap\mathcal F$, suppose every preceding online pull selected the target. Applying the count constraints with $m=N_K^{\rm cl}(t)$ gives
\[
    \begin{aligned}
        \nu_K(t)&>
        \frac{n_K\ru+N_K^{\rm cl}(t)\underline\mu_K}{n_K+N_K^{\rm cl}(t)}
        -\frac{\gamma}{\sqrt{n_K+N_K^{\rm cl}(t)}}\ge z,\\
        \nu_i(t)&<\rl+\frac{B_i}{N_i^{\rm cl}+n_i}
        +\frac{\gamma}{\sqrt{N_i^{\rm cl}+n_i}}\le z
        \qquad(i<K).
    \end{aligned}
\]
And hence $\nu_K(t) > \nu_i(t)$. Induction from the first online round proves the claim.
\end{proof}

\begin{theorem}[Attack Cost Upper Bound and Allocation for TS]\label{thm:attack-cost-ts}
Fix $\delta\in(0,1/2)$. Under~\Cref{asp:scaling-assumptions}, define
\[
    \Lambda_T^{\rm ts}\triangleq
    3\left(\frac{(K-1)\gamma(T)^2T^2}{4R^2}\right)^{1/3}
    =\Theta(S_T).
\]
The following hold for all sufficiently large $T$.

If $\Delta_K=o(S_T/T)$, there exists a threshold for the above allocation satisfying
\[
    \Pr\left(\sum_{i=1}^{K-1}N_i^\on(T+1)=0\right)\ge1-2\delta
    \quad \text{and} \quad
    \mathcal C(\mathbf n)\le(1+o(1))\Lambda_T^{\rm ts}.
\]
Moreover, under $\mathcal E_\delta$,
\[
    n_K=\left(\frac23+o(1)\right)\Lambda_T^{\rm ts}
    \quad \text{and}\quad
    n_i=(1+o(1))\frac{\Lambda_T^{\rm ts}}{3(K-1)} \text{ for } i < K.
\]
\end{theorem}

\begin{proof}[Proof of~\Cref{thm:attack-cost-ts}]
For $\rho=z-\rl=\Theta(S_T/T)$, expanding the count formulas gives, uniformly over clean logs and including rounding,
\[
    n_K=\frac{T(\rho-\varepsilon)_+}{R}
        +O(T\rho^2+\gamma\sqrt T+1),
    \qquad
    n_i=(1+o(1))\frac{\gamma^2}{\rho^2}\quad(i<K).
\]
The target lower bound follows by setting $m=T$ in its constraint; the upper bound uses $(\rho-\varepsilon)m\le T(\rho-\varepsilon)_+$ and $\sqrt{n+m}\le\sqrt T+n/(2\sqrt T)$. The error is $o(S_T)$. Choose
\[
    \rho=\left(\frac{2R(K-1)\gamma^2}{T}\right)^{1/3},
    \qquad
    \frac{T\rho}{R}=\frac23\Lambda_T^{\rm ts},
    \qquad
    \frac{(K-1)\gamma^2}{\rho^2}=\frac13\Lambda_T^{\rm ts}.
\]
Since $(\rho-\varepsilon)_+\le\rho$, the cost bound holds deterministically and $T_0^{\rm cl}+\mathcal C(\mathbf n)=o(T)$. On $\mathcal E_\delta$, $0\le\varepsilon\le\Delta_K=o(\rho)$, giving the stated allocation. Finally,~\Cref{lemma:ts-threshold-attack} and $\Pr(\mathcal E_\delta\cap\mathcal F)\ge1-2\delta$ give the success guarantee.
\end{proof}

\subsection{Numerical Illustration}
\label{app:ts-experiments}

\begin{figure}[t]
\centering
\begin{minipage}[t]{0.49\linewidth}
\centering
\begin{tikzpicture}
\begin{axis}[
    width=0.97\linewidth, height=4.1cm,
    xlabel={Horizon $T$}, ylabel={Attack Cost},
    label style={font=\small}, tick label style={font=\scriptsize},
    tick align=outside, tick pos=left, axis on top,
    grid=none,
    xmode=normal, ymode=normal,ymin=0,xmin=100000, xmax=1000000, xtick={100000,400000,700000,1000000}, scaled x ticks=base 10:-6,
]

\addplot[name path=TShorizonheuristiccostup, draw=none, forget plot] coordinates {(100000,70375.3181764) (250000,126467.873756) (400000,156786.083348) (550000,175055.052774) (700000,189276.921972) (850000,202385.090374) (1000000,211322.778758)};
\addplot[name path=TShorizonheuristiccostlo, draw=none, forget plot] coordinates {(100000,66248.5218236) (250000,118456.046244) (400000,148315.596652) (550000,168064.707226) (700000,183612.438028) (850000,194012.709626) (1000000,204090.741242)};
\addplot[expheuristic, fill opacity=0.12, draw=none, forget plot] fill between[of=TShorizonheuristiccostup and TShorizonheuristiccostlo];
\addplot[color=expheuristic, thick, solid, mark=square, mark size=2.4pt, mark options={solid, fill=none}] coordinates {(100000,68311.92) (250000,122461.96) (400000,152550.84) (550000,171559.88) (700000,186444.68) (850000,198198.9) (1000000,207706.76)};

\addplot[name path=TShorizontotalcostup, draw=none, forget plot] coordinates {(100000,41897.1309106) (250000,74720.4442654) (400000,100988.412527) (550000,124033.601228) (700000,145025.039733) (850000,164544.864365) (1000000,182935.625976)};
\addplot[name path=TShorizontotalcostlo, draw=none, forget plot] coordinates {(100000,41831.7090894) (250000,74637.5157346) (400000,100894.067473) (550000,123929.758772) (700000,144913.800267) (850000,164428.055635) (1000000,182812.254024)};
\addplot[exptotal, fill opacity=0.12, draw=none, forget plot] fill between[of=TShorizontotalcostup and TShorizontotalcostlo];
\addplot[color=exptotal, thick, solid, mark=o, mark size=2.4pt, mark options={solid, fill=none}] coordinates {(100000,41864.42) (250000,74678.98) (400000,100941.24) (550000,123981.68) (700000,144969.42) (850000,164486.46) (1000000,182873.94)};

\end{axis}
\end{tikzpicture}
\par\vspace{2pt}{\small (a) Attack Cost vs. Horizon $T$}
\end{minipage}\hfill
\begin{minipage}[t]{0.49\linewidth}
\centering
\begin{tikzpicture}
\begin{axis}[
    width=0.97\linewidth, height=4.1cm,
    xlabel={Number of Arms $K$}, ylabel={Attack Cost},
    label style={font=\small}, tick label style={font=\scriptsize},
    tick align=outside, tick pos=left, axis on top,
    grid=none,
    xmode=normal, ymode=normal,ymin=0,xmin=5, xmax=50, xtick={5,10,20,30,40,50},
]

\addplot[name path=TSarmsheuristiccostup, draw=none, forget plot] coordinates {(5,96730.5561606) (10,211322.778758) (15,319423.443002) (20,420780.096456) (30,597910.994626) (40,727889.192247) (50,823654.379631)};
\addplot[name path=TSarmsheuristiccostlo, draw=none, forget plot] coordinates {(5,93709.8838394) (10,204090.741242) (15,307453.756998) (20,406845.703544) (30,576171.445374) (40,706485.007753) (50,802131.980369)};
\addplot[expheuristic, fill opacity=0.12, draw=none, forget plot] fill between[of=TSarmsheuristiccostup and TSarmsheuristiccostlo];
\addplot[color=expheuristic, thick, solid, mark=square, mark size=2.4pt, mark options={solid, fill=none}] coordinates {(5,95220.22) (10,207706.76) (15,313438.6) (20,413812.9) (30,587041.22) (40,717187.1) (50,812893.18)};

\addplot[name path=TSarmstotalcostup, draw=none, forget plot] coordinates {(5,138235.306701) (10,182935.625976) (15,213962.086448) (20,238852.44328) (30,278784.134951) (40,311141.090747) (50,339049.651684)};
\addplot[name path=TSarmstotalcostlo, draw=none, forget plot] coordinates {(5,138140.253299) (10,182812.254024) (15,213839.233552) (20,238728.79672) (30,278622.065049) (40,310991.269253) (50,338873.948316)};
\addplot[exptotal, fill opacity=0.12, draw=none, forget plot] fill between[of=TSarmstotalcostup and TSarmstotalcostlo];
\addplot[color=exptotal, thick, solid, mark=o, mark size=2.4pt, mark options={solid, fill=none}] coordinates {(5,138187.78) (10,182873.94) (15,213900.66) (20,238790.62) (30,278703.1) (40,311066.18) (50,338961.8)};

\end{axis}
\end{tikzpicture}
\par\vspace{2pt}{\small (b) Attack Cost vs. Number of Arms $K$}
\end{minipage}

\par\vspace{2pt}
{\scriptsize\begin{tabular}{cc}
\tikz[baseline=-0.6ex]{\draw[expheuristic, thick, solid] (0,0)--(.48,0); \draw plot[mark=square, mark size=2.4pt, mark options={solid, fill=none, draw=expheuristic}] coordinates {(.24,0)};}\;Clipped Suppression & \tikz[baseline=-0.6ex]{\draw[exptotal, thick, solid] (0,0)--(.48,0); \draw plot[mark=o, mark size=2.4pt, mark options={solid, fill=none, draw=exptotal}] coordinates {(.24,0)};}\;Total Attack Cost
\end{tabular}}

\caption{Attack Cost for TS.}
\label{fig:ts_cost_experiments}
\end{figure}
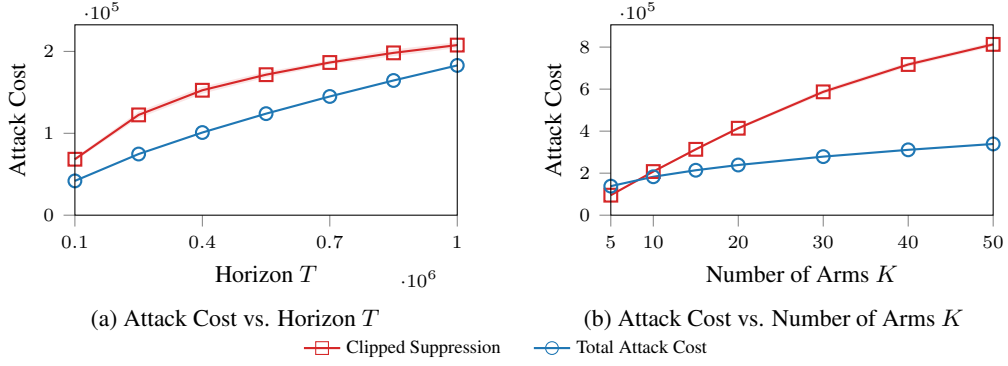
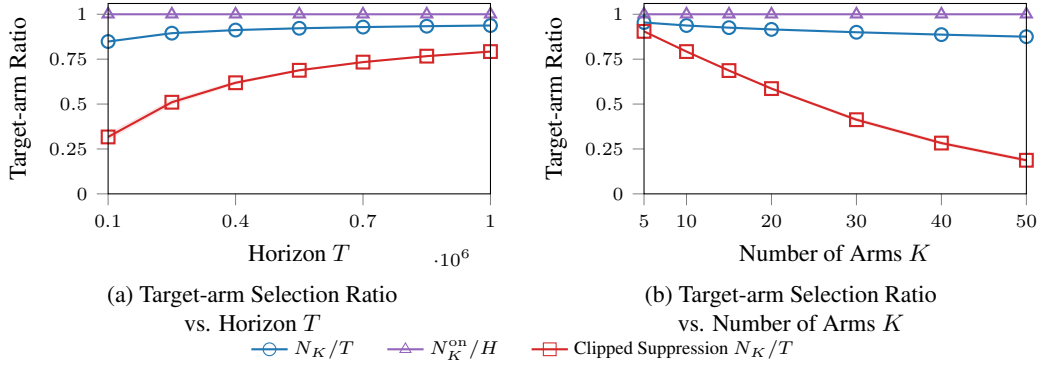
\begin{figure}[t]
\centering
\begin{minipage}[t]{0.49\linewidth}
\centering
\begin{tikzpicture}
\begin{axis}[
    width=0.97\linewidth, height=4.1cm,
    xlabel={Horizon $T$}, ylabel={Target-arm Ratio},
    label style={font=\small}, tick label style={font=\scriptsize},
    tick align=outside, tick pos=left, axis on top,
    grid=none,
    xmode=normal, ymode=normal,ymin=0, ymax=1.06, ytick={0,0.25,0.5,0.75,1},xmin=100000, xmax=1000000, xtick={100000,400000,700000,1000000}, scaled x ticks=base 10:-6,
]

\addplot[name path=TShorizontotaltarget_ratioup, draw=none, forget plot] coordinates {(100000,0.848440890026) (250000,0.894859064343) (400000,0.912183552628) (550000,0.922082673645) (700000,0.928744498977) (850000,0.933645479308) (1000000,0.937462785652)};
\addplot[name path=TShorizontotaltarget_ratiolo, draw=none, forget plot] coordinates {(100000,0.847979109974) (250000,0.894626695657) (400000,0.912020647372) (550000,0.921952453628) (700000,0.928634815308) (850000,0.933551132456) (1000000,0.937377814348)};
\addplot[exptotal, fill opacity=0.12, draw=none, forget plot] fill between[of=TShorizontotaltarget_ratioup and TShorizontotaltarget_ratiolo];
\addplot[color=exptotal, thick, solid, mark=o, mark size=2.4pt, mark options={solid, fill=none}] coordinates {(100000,0.84821) (250000,0.89474288) (400000,0.9121021) (550000,0.922017563636) (700000,0.928689657143) (850000,0.933598305882) (1000000,0.9374203)};

\addplot[name path=TShorizononlineonline_ratioup, draw=none, forget plot] coordinates {(100000,1) (250000,1) (400000,1) (550000,1) (700000,1) (850000,1) (1000000,1)};
\addplot[name path=TShorizononlineonline_ratiolo, draw=none, forget plot] coordinates {(100000,1) (250000,1) (400000,1) (550000,1) (700000,1) (850000,1) (1000000,1)};
\addplot[exponline, fill opacity=0.12, draw=none, forget plot] fill between[of=TShorizononlineonline_ratioup and TShorizononlineonline_ratiolo];
\addplot[color=exponline, thick, solid, mark=triangle, mark size=2.4pt, mark options={solid, fill=none}] coordinates {(100000,1) (250000,1) (400000,1) (550000,1) (700000,1) (850000,1) (1000000,1)};

\addplot[name path=TShorizonheuristictarget_ratioup, draw=none, forget plot] coordinates {(100000,0.337514781764) (250000,0.526175815023) (400000,0.629211008371) (550000,0.694427805044) (700000,0.737696517103) (850000,0.771749753381) (1000000,0.795909258758)};
\addplot[name path=TShorizonheuristictarget_ratiolo, draw=none, forget plot] coordinates {(100000,0.296246818236) (250000,0.494128504977) (400000,0.608034791629) (550000,0.681718085865) (700000,0.729604397182) (850000,0.761899893678) (1000000,0.788677221242)};
\addplot[expheuristic, fill opacity=0.12, draw=none, forget plot] fill between[of=TShorizonheuristictarget_ratioup and TShorizonheuristictarget_ratiolo];
\addplot[color=expheuristic, thick, solid, mark=square, mark size=2.4pt, mark options={solid, fill=none}] coordinates {(100000,0.3168808) (250000,0.51015216) (400000,0.6186229) (550000,0.688072945455) (700000,0.733650457143) (850000,0.766824823529) (1000000,0.79229324)};

\end{axis}
\end{tikzpicture}
\par\vspace{2pt}{\small (a) Target-arm Selection Ratio\\vs. Horizon $T$}
\end{minipage}\hfill
\begin{minipage}[t]{0.49\linewidth}
\centering
\begin{tikzpicture}
\begin{axis}[
    width=0.97\linewidth, height=4.1cm,
    xlabel={Number of Arms $K$}, ylabel={Target-arm Ratio},
    label style={font=\small}, tick label style={font=\scriptsize},
    tick align=outside, tick pos=left, axis on top,
    grid=none,
    xmode=normal, ymode=normal,ymin=0, ymax=1.06, ytick={0,0.25,0.5,0.75,1},xmin=5, xmax=50, xtick={5,10,20,30,40,50},
]

\addplot[name path=TSarmstotaltarget_ratioup, draw=none, forget plot] coordinates {(5,0.954218474667) (10,0.937462785652) (15,0.925521298386) (20,0.915768151604) (30,0.899858386515) (40,0.886694180471) (50,0.875157908033)};
\addplot[name path=TSarmstotaltarget_ratiolo, draw=none, forget plot] coordinates {(5,0.954153645333) (10,0.937377814348) (15,0.925436101614) (20,0.915682208396) (30,0.899745173485) (40,0.886589499529) (50,0.875034051967)};
\addplot[exptotal, fill opacity=0.12, draw=none, forget plot] fill between[of=TSarmstotaltarget_ratioup and TSarmstotaltarget_ratiolo];
\addplot[color=exptotal, thick, solid, mark=o, mark size=2.4pt, mark options={solid, fill=none}] coordinates {(5,0.95418606) (10,0.9374203) (15,0.9254787) (20,0.91572518) (30,0.89980178) (40,0.88664184) (50,0.87509598)};

\addplot[name path=TSarmsonlineonline_ratioup, draw=none, forget plot] coordinates {(5,1) (10,1) (15,1) (20,1) (30,1) (40,1) (50,1)};
\addplot[name path=TSarmsonlineonline_ratiolo, draw=none, forget plot] coordinates {(5,1) (10,1) (15,1) (20,1) (30,1) (40,1) (50,1)};
\addplot[exponline, fill opacity=0.12, draw=none, forget plot] fill between[of=TSarmsonlineonline_ratioup and TSarmsonlineonline_ratiolo];
\addplot[color=exponline, thick, solid, mark=triangle, mark size=2.4pt, mark options={solid, fill=none}] coordinates {(5,1) (10,1) (15,1) (20,1) (30,1) (40,1) (50,1)};

\addplot[name path=TSarmsheuristictarget_ratioup, draw=none, forget plot] coordinates {(5,0.906290116161) (10,0.795909258758) (15,0.692546243002) (20,0.593154296456) (30,0.423828554626) (40,0.293514992247) (50,0.197868019631)};
\addplot[name path=TSarmsheuristictarget_ratiolo, draw=none, forget plot] coordinates {(5,0.903269443839) (10,0.788677221242) (15,0.680576556998) (20,0.579219903544) (30,0.402089005374) (40,0.272110807753) (50,0.176345620369)};
\addplot[expheuristic, fill opacity=0.12, draw=none, forget plot] fill between[of=TSarmsheuristictarget_ratioup and TSarmsheuristictarget_ratiolo];
\addplot[color=expheuristic, thick, solid, mark=square, mark size=2.4pt, mark options={solid, fill=none}] coordinates {(5,0.90477978) (10,0.79229324) (15,0.6865614) (20,0.5861871) (30,0.41295878) (40,0.2828129) (50,0.18710682)};

\end{axis}
\end{tikzpicture}
\par\vspace{2pt}{\small (b) Target-arm Selection Ratio\\vs. Number of Arms $K$}
\end{minipage}

\par\vspace{2pt}
{\scriptsize\begin{tabular}{ccc}
\tikz[baseline=-0.6ex]{\draw[exptotal, thick, solid] (0,0)--(.48,0); \draw plot[mark=o, mark size=2.4pt, mark options={solid, fill=none, draw=exptotal}] coordinates {(.24,0)};}\;$N_K/T$ & \tikz[baseline=-0.6ex]{\draw[exponline, thick, solid] (0,0)--(.48,0); \draw plot[mark=triangle, mark size=2.4pt, mark options={solid, fill=none, draw=exponline}] coordinates {(.24,0)};}\;$N_K^{\on}/H$ & \tikz[baseline=-0.6ex]{\draw[expheuristic, thick, solid] (0,0)--(.48,0); \draw plot[mark=square, mark size=2.4pt, mark options={solid, fill=none, draw=expheuristic}] coordinates {(.24,0)};}\;Clipped Suppression $N_K/T$
\end{tabular}}

\caption{Target-arm Selection Ratio for TS.}
\label{fig:ts_ratio_experiments}
\end{figure}
\begin{figure}[t]
\centering
\begin{minipage}[t]{0.49\linewidth}
\centering
\begin{tikzpicture}
\begin{axis}[
    width=0.97\linewidth, height=4.1cm,
    xlabel={Horizon $T$}, ylabel={$\mathcal C/\Lambda_T^{\rm ts}$},
    label style={font=\small}, tick label style={font=\scriptsize},
    tick align=outside, tick pos=left, axis on top, grid=none,
    xmode=normal, ymode=normal,
    xmin=100000, xmax=10000000,
    xtick={100000,2500000,5000000,7500000,10000000},
    scaled x ticks=base 10:-6, scaled y ticks=false,
    ymin=0.961075259385, ymax=1.36329757908,
    yticklabel style={/pgf/number format/fixed, /pgf/number format/precision=3},
]
\addplot[expdirect, opacity=0.65, thin, solid, no marks, forget plot] coordinates {(100000,1) (10000000,1)};

\addplot[name path=TS5normalized_costup, draw=none, forget plot] coordinates {(100000,1.2502999151) (250000,1.1942944975) (1000000,1.13532907173) (2500000,1.10786535089) (10000000,1.07776537768)};
\addplot[name path=TS5normalized_costlo, draw=none, forget plot] coordinates {(100000,1.24831165726) (250000,1.19292335845) (1000000,1.13454839642) (2500000,1.10731366472) (10000000,1.07743961666)};
\addplot[exptotal, fill opacity=0.12, draw=none, forget plot] fill between[of=TS5normalized_costup and TS5normalized_costlo];
\addplot[color=exptotal, thick, solid, mark=o, mark size=2.4pt, mark options={solid, fill=none}] coordinates {(100000,1.24930578618) (250000,1.19360892797) (1000000,1.13493873407) (2500000,1.1075895078) (10000000,1.07760249717)};

\addplot[name path=TS10normalized_costup, draw=none, forget plot] coordinates {(100000,1.27050930825) (250000,1.2056444767) (1000000,1.13876925928) (2500000,1.10833221356) (10000000,1.07570810685)};
\addplot[name path=TS10normalized_costlo, draw=none, forget plot] coordinates {(100000,1.2685254246) (250000,1.20430639144) (1000000,1.13800127226) (2500000,1.10779059442) (10000000,1.07539048451)};
\addplot[expheuristic, fill opacity=0.12, draw=none, forget plot] fill between[of=TS10normalized_costup and TS10normalized_costlo];
\addplot[color=expheuristic, thick, solid, mark=square, mark size=2.4pt, mark options={solid, fill=none}] coordinates {(100000,1.26951736643) (250000,1.20497543407) (1000000,1.13838526577) (2500000,1.10806140399) (10000000,1.07554929568)};

\addplot[name path=TS15normalized_costup, draw=none, forget plot] coordinates {(100000,1.29033660303) (250000,1.21838046861) (1000000,1.14499150313) (2500000,1.11194173073) (10000000,1.0769274747)};
\addplot[name path=TS15normalized_costlo, draw=none, forget plot] coordinates {(100000,1.28863546216) (250000,1.21723287778) (1000000,1.14433407113) (2500000,1.11149292197) (10000000,1.07665932598)};
\addplot[exponline, fill opacity=0.12, draw=none, forget plot] fill between[of=TS15normalized_costup and TS15normalized_costlo];
\addplot[color=exponline, thick, solid, mark=triangle, mark size=2.4pt, mark options={solid, fill=none}] coordinates {(100000,1.2894860326) (250000,1.21780667319) (1000000,1.14466278713) (2500000,1.11171732635) (10000000,1.07679340034)};

\addplot[name path=TS20normalized_costup, draw=none, forget plot] coordinates {(100000,1.30838051005) (250000,1.23035213993) (1000000,1.1512916715) (2500000,1.11596251505) (10000000,1.07879403198)};
\addplot[name path=TS20normalized_costlo, draw=none, forget plot] coordinates {(100000,1.30686741419) (250000,1.22930539527) (1000000,1.15069568323) (2500000,1.11555160461) (10000000,1.07855316246)};
\addplot[expxu, fill opacity=0.12, draw=none, forget plot] fill between[of=TS20normalized_costup and TS20normalized_costlo];
\addplot[color=expxu, thick, solid, mark=diamond, mark size=2.4pt, mark options={solid, fill=none}] coordinates {(100000,1.30762396212) (250000,1.2298287676) (1000000,1.15099367736) (2500000,1.11575705983) (10000000,1.07867359722)};

\addplot[name path=TS25normalized_costup, draw=none, forget plot] coordinates {(100000,1.32437283846) (250000,1.24109111254) (1000000,1.15712179614) (2500000,1.11981115159) (10000000,1.08074258819)};
\addplot[name path=TS25normalized_costlo, draw=none, forget plot] coordinates {(100000,1.32287612223) (250000,1.24005026741) (1000000,1.15653602653) (2500000,1.11940158866) (10000000,1.08050401237)};
\addplot[expdirect, fill opacity=0.12, draw=none, forget plot] fill between[of=TS25normalized_costup and TS25normalized_costlo];
\addplot[color=expdirect, thick, solid, mark=x, mark size=2.4pt, mark options={solid, fill=none}] coordinates {(100000,1.32362448035) (250000,1.24057068998) (1000000,1.15682891133) (2500000,1.11960637013) (10000000,1.08062330028)};

\end{axis}
\end{tikzpicture}
\par\vspace{2pt}
{\scriptsize\setlength{\tabcolsep}{1pt}\begin{tabular}{ccccc}
\tikz[baseline=-0.6ex]{\draw[exptotal, thick, solid] (0,0)--(.30,0); \draw plot[mark=o, mark size=2.4pt, mark options={solid, fill=none, draw=exptotal}] coordinates {(.15,0)};}\;$K=5$ & \tikz[baseline=-0.6ex]{\draw[expheuristic, thick, solid] (0,0)--(.30,0); \draw plot[mark=square, mark size=2.4pt, mark options={solid, fill=none, draw=expheuristic}] coordinates {(.15,0)};}\;$K=10$ & \tikz[baseline=-0.6ex]{\draw[exponline, thick, solid] (0,0)--(.30,0); \draw plot[mark=triangle, mark size=2.4pt, mark options={solid, fill=none, draw=exponline}] coordinates {(.15,0)};}\;$K=15$ & \tikz[baseline=-0.6ex]{\draw[expxu, thick, solid] (0,0)--(.30,0); \draw plot[mark=diamond, mark size=2.4pt, mark options={solid, fill=none, draw=expxu}] coordinates {(.15,0)};}\;$K=20$ & \tikz[baseline=-0.6ex]{\draw[expdirect, thick, solid] (0,0)--(.30,0); \draw plot[mark=x, mark size=2.4pt, mark options={solid, fill=none, draw=expdirect}] coordinates {(.15,0)};}\;$K=25$
\end{tabular}}

\par\vspace{2pt}{\small (a) Normalized Attack Cost vs. Horizon $T$}
\end{minipage}\hfill
\begin{minipage}[t]{0.49\linewidth}
\centering
\begin{tikzpicture}
\begin{axis}[
    width=0.97\linewidth, height=4.1cm,
    xlabel={Horizon $T$}, ylabel={Cost Share},
    label style={font=\small}, tick label style={font=\scriptsize},
    tick align=outside, tick pos=left, axis on top, grid=none,
    xmode=normal, ymode=normal,
    xmin=100000, xmax=10000000,
    xtick={100000,2500000,5000000,7500000,10000000},
    scaled x ticks=base 10:-6, scaled y ticks=false,
    ymin=0.3, ymax=0.7,
    ytick={0.333333333333,0.666666666667}, yticklabels={$1/3$,$2/3$},
    yticklabel style={/pgf/number format/fixed, /pgf/number format/precision=3},
]
\addplot[expdirect, opacity=0.65, thin, solid, no marks, forget plot] coordinates {(100000,0.666666666667) (10000000,0.666666666667)};
\addplot[expdirect, opacity=0.65, thin, solid, no marks, forget plot] coordinates {(100000,0.333333333333) (10000000,0.333333333333)};

\addplot[name path=TStargetallocation_fractionsup, draw=none, forget plot] coordinates {(100000,0.638769265363) (250000,0.648431153632) (1000000,0.658161861525) (2500000,0.662282428409) (10000000,0.666188867716)};
\addplot[name path=TStargetallocation_fractionslo, draw=none, forget plot] coordinates {(100000,0.638230565405) (250000,0.64804364747) (1000000,0.657927758139) (2500000,0.662112643097) (10000000,0.666086999598)};
\addplot[exptarget, fill opacity=0.12, draw=none, forget plot] fill between[of=TStargetallocation_fractionsup and TStargetallocation_fractionslo];
\addplot[color=exptarget, thick, solid, mark=triangle, mark size=2.4pt, mark options={solid, fill=none}] coordinates {(100000,0.638499915384) (250000,0.648237400551) (1000000,0.658044809832) (2500000,0.662197535753) (10000000,0.666137933657)};

\addplot[name path=TSnon_targetallocation_fractionsup, draw=none, forget plot] coordinates {(100000,0.361769434595) (250000,0.35195635253) (1000000,0.342072241861) (2500000,0.337887356903) (10000000,0.333913000402)};
\addplot[name path=TSnon_targetallocation_fractionslo, draw=none, forget plot] coordinates {(100000,0.361230734637) (250000,0.351568846368) (1000000,0.341838138475) (2500000,0.337717571591) (10000000,0.333811132284)};
\addplot[expnon_target, fill opacity=0.12, draw=none, forget plot] fill between[of=TSnon_targetallocation_fractionsup and TSnon_targetallocation_fractionslo];
\addplot[color=expnon_target, thick, solid, mark=square, mark size=2.4pt, mark options={solid, fill=none}] coordinates {(100000,0.361500084616) (250000,0.351762599449) (1000000,0.341955190168) (2500000,0.337802464247) (10000000,0.333862066343)};

\end{axis}
\end{tikzpicture}
\par\vspace{2pt}
{\scriptsize\begin{tabular}{cc}
\tikz[baseline=-0.6ex]{\draw[exptarget, thick, solid] (0,0)--(.48,0); \draw plot[mark=triangle, mark size=2.4pt, mark options={solid, fill=none, draw=exptarget}] coordinates {(.24,0)};}\;Target $n_K/\mathcal C$ & \tikz[baseline=-0.6ex]{\draw[expnon_target, thick, solid] (0,0)--(.48,0); \draw plot[mark=square, mark size=2.4pt, mark options={solid, fill=none, draw=expnon_target}] coordinates {(.24,0)};}\;Non-target $\sum_{i<K}n_i/\mathcal C$
\end{tabular}}

\par\vspace{2pt}{\small (b) Cost Share vs. Horizon $T$}
\end{minipage}

\caption{Near-Boundary Cost and Allocation for TS.}
\label{fig:ts-theory-alignment}
\end{figure}
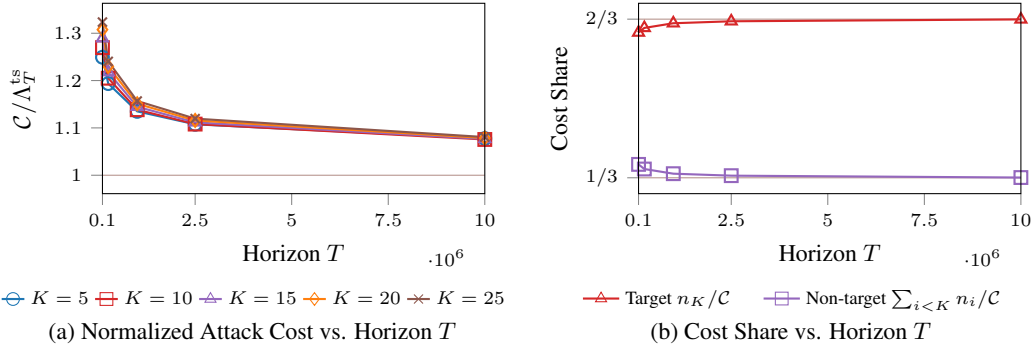

We repeat the Clipped Suppression comparisons and near-boundary experiments of~\Cref{sec:exp} with TS, and numerically minimize the continuous total count defined by Eq.~\eqref{eq:ts-closed-budget} and Eq.~\eqref{eq:ts-target-constraint} before rounding. All other settings and reporting conventions are unchanged.


Figures~\ref{fig:ts_cost_experiments} and~\ref{fig:ts_ratio_experiments} show exclusively target online pulls for our attack in every repeat. Its cost is lower than Clipped Suppression except at the smallest arm count, where the baseline costs less but still pulls non-target arms.

Figure~\ref{fig:ts-theory-alignment} shows normalized cost approaching one and cost shares approaching $2/3$ for the target and $1/3$ for the non-target arms collectively, consistent with the TS construction. Every repeat has zero non-target online pulls.

\section{Analysis for \texorpdfstring{$\epsilon$}{epsilon}-Greedy}\label{app:analysis-EG}

\subsection{Attack Strategy}\label{app:proof-eg-upper-bound}

At round $s$, $\epsilon$-greedy explores with probability $\epsilon_s$ by choosing uniformly from all $K$ arms; otherwise it selects an empirical-mean maximizer. The schedule $\{\epsilon_s\}_{s=1}^T\subset[0,1]$ is deterministic, and exploration randomization is independent across rounds.

For simplicity, assume that the entire offline log is constructed by the attacker, so $N_i^{\rm cl}=0$ for every arm and $T_0=\mathcal C(\mathbf n)$.

For $0\le t<T$ and $\delta\in(0,1)$, define the exploration bound
\[
    B_i^{\rm eg}(t,T,\delta)=\sum_{s=t+1}^T\frac{\epsilon_s}{K}
        +\sqrt{2\left(\sum_{s=t+1}^T\frac{\epsilon_s}{K}\right)\log\frac1\delta}
        +\frac13\log\frac1\delta.
\]
\begin{lemma}[Corollary C.1 of~\citet{zeng2025practical}]\label[lemma]{lem:eg-exploration-counts}
Fix $T_0\le t<T$ and $\delta\in(0,1)$. With probability at least $1-\delta$, the following implication holds: if $\hat\mu_i(s)<\max_j\hat\mu_j(s)$ for all $s\in\{t+1,\ldots,T\}$, then
\[
    N_i^\on(T+1)-N_i^\on(t+1)\le B_i^{\rm eg}(t,T,\delta).
\]
\end{lemma}

Fix $\delta\in(0,1/2)$, write $B=\lceil B_i^{\rm eg}(0,T,\delta/(K-1))\rceil$ for short, and define
\[
    r_T^{\rm eg}=R\sqrt{\frac{(K-1)B}{T}}.
\]
Inject reward $\ru$ on the target and $\rl$ on each non-target arm, with
\[
    n_K=\left\lceil\sqrt{T(K-1)B}\right\rceil+1,
    \qquad
    n_i=\left\lceil\sqrt{\frac{TB}{K-1}}\right\rceil\quad(i<K).
\]

\begin{theorem}[$\epsilon$-Greedy Upper Bound]\label{thm:attack-cost-eg}
Fix $\delta\in(0,1/2)$. Suppose $R=\Theta(1)$,
$\sum_{s=1}^T\epsilon_s+K\log(K/\delta)=o(T)$.
For all sufficiently large $T$, the above attack satisfies
\[
    \Pr\left(\sum_{i<K}N_i^\on(T+1)\le(K-1)B\right)\ge1-\delta,
    \qquad
    \mathcal C(\mathbf n)=(2+o(1))\sqrt{T(K-1)B}.
\]
All greedy online decisions select the target on the same event.
\end{theorem}
\begin{proof}
The definition of $B$ gives $(K-1)B=O(\sum_s\epsilon_s+K\log(K/\delta))=o(T)$, while rounding yields
\[
    2\sqrt{T(K-1)B}<\mathcal C(\mathbf n)
    \le2\sqrt{T(K-1)B}+K+1.
\]
Since $B\ge1$ and $K=o(T)$, the rounding term is lower order, proving the cost bound and $T_0=\mathcal C(\mathbf n)<T$ eventually.

The Bernoulli bound underlying~\Cref{lem:eg-exploration-counts} and a union bound ensure at most $B$ potential exploratory selections per non-target arm over rounds $1,\ldots,T$, with probability at least $1-\delta$. On this event, if preceding greedy decisions selected the target, bounded rewards and $n_Kn_i>TB$ give
\[
    \hat\mu_i(s)\le\rl+\frac{RB}{n_i+B}
    <\rl+\frac{Rn_K}{n_K+T}
    \le\hat\mu_K(s).
\]
Induction from the first online round makes every greedy decision select the target, so all non-target pulls are exploratory and total at most $(K-1)B$.
\end{proof}

For $\epsilon_s=\min\{1,cK/s\}$ with fixed $c>0$ and $K\log T=o(T)$, this gives $r_T^{\rm eg}=\Theta(\sqrt{K\log T/T})$, cost $O(\sqrt{TK\log T})$, and $(\xi,\delta)$-success with $\xi(T)=O(K\log T)$.

\subsection{Numerical Illustration}\label{app:eg-experiments}

We use Bernoulli rewards with non-target means equally spaced in $[0.2,0.8]$ and target mean $r_T^{\rm eg}/\log T>0$, with $\epsilon_s=\min\{1,K/s\}$ and $\delta=0.05$. Online statistics use 50 independent runs per configuration; shaded bands show one sample standard deviation. Greedy ties are broken by arm index.

\textit{Clipped Suppression} sets every non-target online reward to zero and counts each such round as attack cost. It has no offline log and initializes with one online pull per arm, starting with the target.

Figures~\ref{fig:eg_cost_experiments} and~\ref{fig:eg_ratio_experiments} vary $T$ at $K=10$ and $K$ at $T=10^6$. Our attack has lower mean costs and higher mean target-arm selection ratios. All greedy decisions selected the target; the remaining non-target pulls were exploratory and stayed below the theoretical bound.


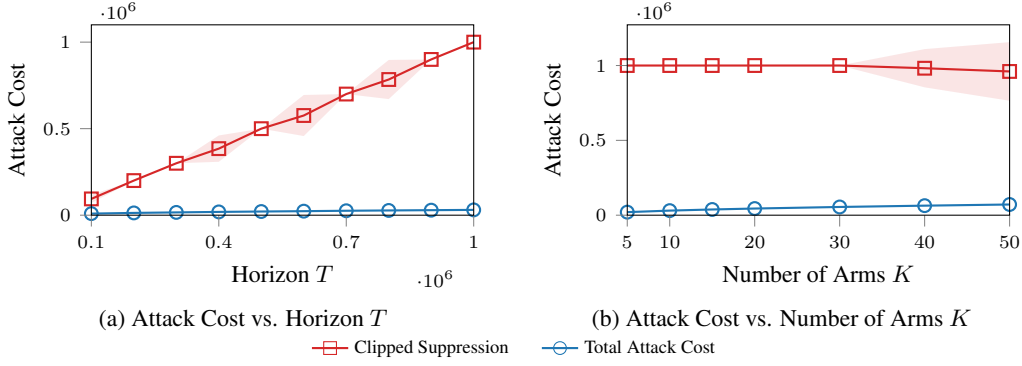
\begin{figure}[t]
\centering
\begin{minipage}[t]{0.49\linewidth}
\centering
\begin{tikzpicture}
\begin{axis}[
    width=0.97\linewidth, height=4.1cm,
    xlabel={Horizon $T$}, ylabel={Attack Cost},
    label style={font=\small}, tick label style={font=\scriptsize},
    tick align=outside, tick pos=left, axis on top, grid=none,
    xmode=normal, ymode=normal, ymin=0,xmin=100000, xmax=1000000, xtick={100000,400000,700000,1000000}, scaled x ticks=base 10:-6,
]

\addplot[name path=eghorizonheuristiccostup, draw=none, forget plot] coordinates {(100000,117563.501744) (200000,199992.768564) (300000,299991.977302) (400000,460782.598652) (500000,499992.106644) (600000,694739.334575) (700000,699990.699713) (800000,897110.841855) (900000,899991.033951) (1000000,999990.590714)};
\addplot[name path=eghorizonheuristiccostlo, draw=none, forget plot] coordinates {(100000,70697.7382564) (200000,199986.391436) (300000,299985.342698) (400000,308433.041348) (500000,499986.013356) (600000,457246.705425) (700000,699984.980287) (800000,670870.278145) (900000,899985.286049) (1000000,999983.729286)};
\addplot[expheuristic, fill opacity=0.12, draw=none, forget plot] fill between[of=eghorizonheuristiccostup and eghorizonheuristiccostlo];
\addplot[color=expheuristic, thick, solid, mark=square, mark size=2.4pt, mark options={solid, fill=none}] coordinates {(100000,94130.62) (200000,199989.58) (300000,299988.66) (400000,384607.82) (500000,499989.06) (600000,575993.02) (700000,699987.84) (800000,783990.56) (900000,899988.16) (1000000,999987.16)};
\addplot[color=exptotal, thick, solid, mark=o, mark size=2.4pt, mark options={solid, fill=none}] coordinates {(100000,9105) (200000,13153) (300000,16106) (400000,18983) (500000,21219) (600000,23239) (700000,25607) (800000,27373) (900000,29031) (1000000,30599)};
\end{axis}
\end{tikzpicture}
\par\vspace{2pt}{\small (a) Attack Cost vs. Horizon $T$}
\end{minipage}\hfill
\begin{minipage}[t]{0.49\linewidth}
\centering
\begin{tikzpicture}
\begin{axis}[
    width=0.97\linewidth, height=4.1cm,
    xlabel={Number of Arms $K$}, ylabel={Attack Cost},
    label style={font=\small}, tick label style={font=\scriptsize},
    tick align=outside, tick pos=left, axis on top, grid=none,
    xmode=normal, ymode=normal, ymin=0,xmin=5, xmax=50, xtick={5,10,20,30,40,50},
]

\addplot[name path=egarmsheuristiccostup, draw=none, forget plot] coordinates {(5,999991.475691) (10,999990.590714) (15,999992.052136) (20,999991.496901) (30,999991.120952) (40,1109675.31342) (50,1156329.00182)};
\addplot[name path=egarmsheuristiccostlo, draw=none, forget plot] coordinates {(5,999983.524309) (10,999983.729286) (15,999985.027864) (20,999985.583099) (30,999985.599048) (40,854170.526583) (50,764467.198177)};
\addplot[expheuristic, fill opacity=0.12, draw=none, forget plot] fill between[of=egarmsheuristiccostup and egarmsheuristiccostlo];
\addplot[color=expheuristic, thick, solid, mark=square, mark size=2.4pt, mark options={solid, fill=none}] coordinates {(5,999987.5) (10,999987.16) (15,999988.54) (20,999988.54) (30,999988.36) (40,981922.92) (50,960398.1)};
\addplot[color=exptotal, thick, solid, mark=o, mark size=2.4pt, mark options={solid, fill=none}] coordinates {(5,20400) (10,30599) (15,38162) (20,44458) (30,54924) (40,63708) (50,71416)};
\end{axis}
\end{tikzpicture}
\par\vspace{2pt}{\small (b) Attack Cost vs. Number of Arms $K$}
\end{minipage}

\par\vspace{2pt}
{\scriptsize\begin{tabular}{cc}
\tikz[baseline=-0.6ex]{\draw[expheuristic, thick, solid] (0,0)--(.48,0); \draw plot[mark=square, mark size=2.4pt, mark options={solid, fill=none, draw=expheuristic}] coordinates {(.24,0)};}\;Clipped Suppression & \tikz[baseline=-0.6ex]{\draw[exptotal, thick, solid] (0,0)--(.48,0); \draw plot[mark=o, mark size=2.4pt, mark options={solid, fill=none, draw=exptotal}] coordinates {(.24,0)};}\;Total Attack Cost
\end{tabular}}

\caption{Attack Cost for $\epsilon$-greedy.}
\label{fig:eg_cost_experiments}
\end{figure}
\begin{figure}[t]
\centering
\begin{minipage}[t]{0.49\linewidth}
\centering
\begin{tikzpicture}
\begin{axis}[
    width=0.97\linewidth, height=4.1cm,
    xlabel={Horizon $T$}, ylabel={Target-arm Ratio},
    label style={font=\small}, tick label style={font=\scriptsize},
    tick align=outside, tick pos=left, axis on top, grid=none,
    xmode=normal, ymode=normal, ymin=0,ymax=1.06, ytick={0,0.25,0.5,0.75,1},xmin=100000, xmax=1000000, xtick={100000,400000,700000,1000000}, scaled x ticks=base 10:-6,
]

\addplot[name path=eghorizontotaltotal_target_ratioup, draw=none, forget plot] coordinates {(100000,0.954291201874) (200000,0.967003706326) (300000,0.973075132628) (400000,0.976211343562) (500000,0.978734539925) (600000,0.98059567479) (700000,0.981669626692) (800000,0.982855535527) (900000,0.983842000563) (1000000,0.984674033228)};
\addplot[name path=eghorizontotaltotal_target_ratiolo, draw=none, forget plot] coordinates {(100000,0.954193198126) (200000,0.966953693674) (300000,0.973046067372) (400000,0.976181756438) (500000,0.978709940075) (600000,0.980580791876) (700000,0.981656430451) (800000,0.982843564473) (900000,0.983827243882) (1000000,0.984664166772)};
\addplot[exptotal, fill opacity=0.12, draw=none, forget plot] fill between[of=eghorizontotaltotal_target_ratioup and eghorizontotaltotal_target_ratiolo];
\addplot[color=exptotal, thick, solid, mark=o, mark size=2.4pt, mark options={solid, fill=none}] coordinates {(100000,0.9542422) (200000,0.9669787) (300000,0.9730606) (400000,0.97619655) (500000,0.97872224) (600000,0.980588233333) (700000,0.981663028571) (800000,0.98284955) (900000,0.983834622222) (1000000,0.9846691)};

\addplot[name path=eghorizononlineonline_ratioup, draw=none, forget plot] coordinates {(100000,0.999814293277) (200000,0.999891575809) (300000,0.999920885219) (400000,0.99994629485) (500000,0.999954613827) (600000,0.999959090289) (700000,0.999964024959) (800000,0.999965608788) (900000,0.99997106729) (1000000,0.99997321359)};
\addplot[name path=eghorizononlineonline_ratiolo, draw=none, forget plot] coordinates {(100000,0.999706472442) (200000,0.999838042542) (300000,0.99989017102) (400000,0.999915233638) (500000,0.99992892374) (600000,0.999943607709) (700000,0.999950327651) (800000,0.999953213619) (900000,0.999955818742) (1000000,0.999963035702)};
\addplot[exponline, fill opacity=0.12, draw=none, forget plot] fill between[of=eghorizononlineonline_ratioup and eghorizononlineonline_ratiolo];
\addplot[color=exponline, thick, solid, mark=triangle, mark size=2.4pt, mark options={solid, fill=none}] coordinates {(100000,0.999760382859) (200000,0.999864809175) (300000,0.99990552812) (400000,0.999930764244) (500000,0.999941768784) (600000,0.999951348999) (700000,0.999957176305) (800000,0.999959411204) (900000,0.999963443016) (1000000,0.999968124646)};

\addplot[name path=eghorizonheuristictotal_target_ratioup, draw=none, forget plot] coordinates {(100000,0.293022617436) (200000,6.80428187404e-05) (300000,4.8857671945e-05) (400000,0.228917396629) (500000,2.79732883786e-05) (600000,0.237922157625) (700000,2.14567330423e-05) (800000,0.161412152319) (900000,1.63488339585e-05) (1000000,1.62707136165e-05)};
\addplot[name path=eghorizonheuristictotal_target_ratiolo, draw=none, forget plot] coordinates {(100000,0) (200000,3.61571812596e-05) (300000,2.6742328055e-05) (400000,0) (500000,1.57867116214e-05) (600000,0) (700000,1.32861241006e-05) (800000,0) (900000,9.96227715258e-06) (1000000,9.40928638351e-06)};
\addplot[expheuristic, fill opacity=0.12, draw=none, forget plot] fill between[of=eghorizonheuristictotal_target_ratioup and eghorizonheuristictotal_target_ratiolo];
\addplot[color=expheuristic, thick, solid, mark=square, mark size=2.4pt, mark options={solid, fill=none}] coordinates {(100000,0.0586938) (200000,5.21e-05) (300000,3.78e-05) (400000,0.03848045) (500000,2.188e-05) (600000,0.0400116333333) (700000,1.73714285714e-05) (800000,0.0200118) (900000,1.31555555556e-05) (1000000,1.284e-05)};
\end{axis}
\end{tikzpicture}
\par\vspace{2pt}{\small (a) Target-arm Selection Ratio\\vs. Horizon $T$}
\end{minipage}\hfill
\begin{minipage}[t]{0.49\linewidth}
\centering
\begin{tikzpicture}
\begin{axis}[
    width=0.97\linewidth, height=4.1cm,
    xlabel={Number of Arms $K$}, ylabel={Target-arm Ratio},
    label style={font=\small}, tick label style={font=\scriptsize},
    tick align=outside, tick pos=left, axis on top, grid=none,
    xmode=normal, ymode=normal, ymin=0,ymax=1.06, ytick={0,0.25,0.5,0.75,1},xmin=5, xmax=50, xtick={5,10,20,30,40,50},
]

\addplot[name path=egarmstotaltotal_target_ratioup, draw=none, forget plot] coordinates {(5,0.989786998343) (10,0.984674033228) (15,0.980878523165) (20,0.977716811578) (30,0.972460018128) (40,0.96803884894) (50,0.964159615386)};
\addplot[name path=egarmstotaltotal_target_ratiolo, draw=none, forget plot] coordinates {(5,0.989781561657) (10,0.984664166772) (15,0.980864756835) (20,0.977703908422) (30,0.972444141872) (40,0.96801955106) (50,0.964136504614)};
\addplot[exptotal, fill opacity=0.12, draw=none, forget plot] fill between[of=egarmstotaltotal_target_ratioup and egarmstotaltotal_target_ratiolo];
\addplot[color=exptotal, thick, solid, mark=o, mark size=2.4pt, mark options={solid, fill=none}] coordinates {(5,0.98978428) (10,0.9846691) (15,0.98087164) (20,0.97771036) (30,0.97245208) (40,0.9680292) (50,0.96414806)};

\addplot[name path=egarmsonlineonline_ratioup, draw=none, forget plot] coordinates {(5,0.999986727586) (10,0.99997321359) (15,0.999958956877) (20,0.999944336908) (30,0.999918544253) (40,0.99989517046) (50,0.999871433694)};
\addplot[name path=egarmsonlineonline_ratiolo, draw=none, forget plot] coordinates {(5,0.999981177682) (10,0.999963035702) (15,0.999944644353) (20,0.999930833414) (30,0.999901745333) (40,0.999874559497) (50,0.999846545508)};
\addplot[exponline, fill opacity=0.12, draw=none, forget plot] fill between[of=egarmsonlineonline_ratioup and egarmsonlineonline_ratiolo];
\addplot[color=exponline, thick, solid, mark=triangle, mark size=2.4pt, mark options={solid, fill=none}] coordinates {(5,0.999983952634) (10,0.999968124646) (15,0.999951800615) (20,0.999937585161) (30,0.999910144793) (40,0.999884864978) (50,0.999858989601)};

\addplot[name path=egarmsheuristictotal_target_ratioup, draw=none, forget plot] coordinates {(5,1.64756914429e-05) (10,1.62707136165e-05) (15,1.49721363927e-05) (20,1.44169013014e-05) (30,1.44009522167e-05) (40,0.145829473417) (50,0.235532801823)};
\addplot[name path=egarmsheuristictotal_target_ratiolo, draw=none, forget plot] coordinates {(5,8.52430855712e-06) (10,9.40928638351e-06) (15,7.94786360732e-06) (20,8.50309869862e-06) (30,8.87904778331e-06) (40,0) (50,0)};
\addplot[expheuristic, fill opacity=0.12, draw=none, forget plot] fill between[of=egarmsheuristictotal_target_ratioup and egarmsheuristictotal_target_ratiolo];
\addplot[color=expheuristic, thick, solid, mark=square, mark size=2.4pt, mark options={solid, fill=none}] coordinates {(5,1.25e-05) (10,1.284e-05) (15,1.146e-05) (20,1.146e-05) (30,1.164e-05) (40,0.01807708) (50,0.0396019)};
\end{axis}
\end{tikzpicture}
\par\vspace{2pt}{\small (b) Target-arm Selection Ratio\\vs. Number of Arms $K$}
\end{minipage}

\par\vspace{2pt}
{\scriptsize\begin{tabular}{ccc}
\tikz[baseline=-0.6ex]{\draw[exptotal, thick, solid] (0,0)--(.48,0); \draw plot[mark=o, mark size=2.4pt, mark options={solid, fill=none, draw=exptotal}] coordinates {(.24,0)};}\;$N_K/T$ & \tikz[baseline=-0.6ex]{\draw[exponline, thick, solid] (0,0)--(.48,0); \draw plot[mark=triangle, mark size=2.4pt, mark options={solid, fill=none, draw=exponline}] coordinates {(.24,0)};}\;$N_K^{\on}/H$ & \tikz[baseline=-0.6ex]{\draw[expheuristic, thick, solid] (0,0)--(.48,0); \draw plot[mark=square, mark size=2.4pt, mark options={solid, fill=none, draw=expheuristic}] coordinates {(.24,0)};}\;Clipped Suppression $N_K/T$
\end{tabular}}

\caption{Target-arm Selection Ratio for $\epsilon$-greedy.}
\label{fig:eg_ratio_experiments}
\end{figure}
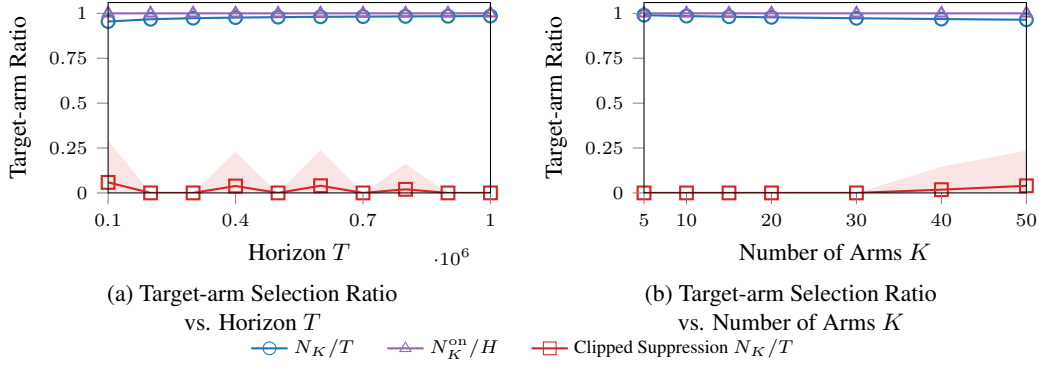

\section{Analysis of Suppression-Bounded Algorithms}\label{app:analysis-suppression-bounded}

\begin{proof}[Proof of~\Cref{thm:pull-count-upper}]
Fix $\rho\in(0,R)$ and write $B_i^{\mathcal A}=B_i^{\mathcal A}(T_0^{\rm cl},T,\delta/K)$.
Choose \[
    n_K = \left\lceil T\frac{(\rho-\varepsilon)_+}{R-\rho}\right\rceil+1, \quad n_i = \left\lceil
        \frac{R - \rho}{\rho}
        \bigl(N_i^{\rm cl} + B_i^{\mathcal A}\bigr)
    \right\rceil, \text{ for all } i < K.
\]
On $\mathcal E_\delta$, every empirical average of the genuine target-arm rewards is at least $\underline \mu_K = \rl + \varepsilon$. Thus, for any $t \in \mathcal T^\on$,
\[
    \hat \mu_K(t) \ge \frac{n_K \ru + N_K^{\rm cl}(t) \underline \mu_K}{n_K + N_K^{\rm cl}(t)}>\rl+\rho.
\]

    
For any bandit algorithm $\mathcal A$, the pull count $N_i(t)$, excluding injected samples, is nondecreasing in $t$. Thus, for each arm $i < K$, we use the conservative upper bound $B_i^{\mathcal A}$ uniformly over $t \in [T_0^{\rm cl}+1,T]$.

We now use induction over online rounds. Initially, $N_i^{\on}(T_0 + 1) = 0$. At any later round $t$, if arm $i$ has been suppressed at every preceding online decision, we have $N_i^{\on}(t) \le B_i^{\mathcal A}$. In either case, $N_i^{\rm cl} (t) = N_i^{\rm cl} + N_i^{\on} (t) \le N_i^{\rm cl} + B_i^{\mathcal A}$, so, since all genuine rewards are at most $\ru$,
\[
    \hat\mu_i(t) \le \frac{\ru N_i^{\rm cl}(t)+\rl n_i}{N_i^{\rm cl}(t)+n_i} \le \rl+\rho.
\]
Hence
\[
    \hat\mu_i(t)\le \rl+\rho<\hat\mu_K(t).
\]
Intersecting $\mathcal E_\delta$ with the $K-1$ pull-bound events gives probability at least $1 - 2\delta$ by a union bound. On this event,
\[
    \sum_{i=1}^{K-1} N_i^{\on} (T+1)\le\sum_{i=1}^{K-1} B_i^{\mathcal A}.
\]
Summing the displayed injection counts and taking the infimum yields the claimed cost bound, including an $O(K)$ rounding overhead.
\end{proof}
\section{Exploration and Allocation}
\label{app:allocation}

The allocation depends on how injections control exploration. Let $\rho=z-\rl$ and $M=\sum_{i<K}n_i$. In the near-boundary regime, the genuine target rewards contribute little to maintaining the threshold, so the injections must sustain a mean increase of approximately $\rho$ over $T$ observations. This gives the leading target cost $n_K\sim T\rho/R$. The same expression follows from the worst-case reward bound used in our $\epsilon$-greedy construction.

\textbf{UCB and TS.}
For UCB, suppressing a competitor requires its exploration bonus to fit below the target's threshold. Since the bonus decays as $N_i^{-1/2}$, a threshold height $\rho$ requires an observation count proportional to $\rho^{-2}$. Injected samples supply this count before online learning begins. TS has the same dependence because its sampling uncertainty also decays as $N_i^{-1/2}$. Because all competitors share the threshold $\rho$, summing their leading injection costs changes only the coefficient of $\rho^{-2}$, so $M=\sum_{i<K}n_i$ has the same dependence. Thus, for both constructions, raising $\rho$ increases target cost linearly but reduces aggregate non-target cost as $\rho^{-2}$. A small relative increase in $\rho$ adds that fraction of $n_K$ and saves twice that fraction of $M$. At the minimizing threshold these amounts agree, giving $n_K\sim2M$: the target accounts for $2/3$ of the total cost.

\textbf{$\epsilon$-greedy.}
Here exploration is prescribed by the schedule and cannot be removed by increasing observation counts. Each competitor can receive up to $B$ exploratory rewards. Even if all equal $\ru$, injecting $n_i$ rewards equal to $\rl$ keeps its mean below the threshold whenever
\[
    \frac{RB}{n_i+B}\le\rho.
\]
For small $\rho$, this requires approximately $RB/\rho$ injections per competitor. Hence $M$ decreases as $\rho^{-1}$: a small relative increase in $\rho$ saves only that same fraction of $M$. Balancing this saving against the increase in target cost gives $n_K\sim M$, so the target accounts for $1/2$ of the total cost.

\end{document}